\documentclass[sn-mathphys-num]{sn-jnl}

\usepackage{mathrsfs}  

\usepackage{subcaption}
\usepackage{natbib}
\usepackage[utf8]{inputenc} 
\usepackage[T1]{fontenc}  
\usepackage{hyperref}   
\usepackage{url}      
\usepackage{booktabs}   
\usepackage{amsfonts}   
\usepackage{nicefrac}   
\usepackage{microtype}   
\usepackage{xcolor}    
\usepackage{placeins} 
\usepackage{microtype}
\usepackage{graphicx}
\usepackage{subcaption}
\usepackage{booktabs} 
\usepackage{hyperref}

\usepackage{amsmath}
\usepackage{amssymb}
\usepackage{mathtools}
\usepackage{amsthm}
\usepackage{amsfonts,bm}
\usepackage{algorithm,algpseudocode}
\usepackage[capitalize,noabbrev]{cleveref}
\usepackage[textsize=tiny]{todonotes}

\usepackage[utf8]{inputenc} 
\usepackage[T1]{fontenc}    
\usepackage{hyperref}       
\usepackage{url}            
\usepackage{booktabs}       
\usepackage{amsfonts}       
\usepackage{nicefrac}       
\usepackage{microtype}      
\usepackage{xcolor} 
\usepackage{graphicx}%
\usepackage{multirow}%
\usepackage{amsmath,amssymb,amsfonts}%
\usepackage{amsthm}%
\usepackage{mathrsfs}%
\usepackage[title]{appendix}%
\usepackage{xcolor}%
\usepackage{textcomp}%
\usepackage{manyfoot}%
\usepackage{booktabs}%
\usepackage{algorithm}%
\usepackage{algorithmicx}%
\usepackage{algpseudocode}%
\usepackage{listings}%
\usepackage{subcaption} 

\theoremstyle{thmstyleone}%
\newtheorem{theorem}{Theorem}
\newtheorem{proposition}[theorem]{Proposition}%

\theoremstyle{thmstyletwo}%
\newtheorem{example}{Example}%
\newtheorem{remark}{Remark}%

\theoremstyle{thmstylethree}%
\newtheorem{definition}{Definition}%

\newcommand{\R}{\mathbb{R}}

\def\eqref#1{equation~\ref{#1}}

\def\1{\bm{1}}

\DeclareMathAlphabet{\mathsfit}{\encodingdefault}{\sfdefault}{m}{sl}
\SetMathAlphabet{\mathsfit}{bold}{\encodingdefault}{\sfdefault}{bx}{n}

\begin{document}

\title[Affine Invariance in Continuous-Domain \\ Convolutional Neural Networks]{Affine Invariance in Continuous-Domain \\ Convolutional Neural Networks}


\author*[1]{\fnm{Ali} \sur{Mohaddes}}\email{ali.mohaddes@uni-hamburg.de}

\author[1]{\fnm{Johannes} \sur{Lederer}}\email{johannes.lederer@uni-hamburg.de}


\affil*[1]{\orgdiv{Department of Mathematics}, \orgname{University of Hamburg}, \orgaddress{\street{ Mittelweg}, \city{Hamburg}, \postcode{20148}, \state{Hamburg}, \country{Germany}}}




\abstract{		The notion of group invariance helps neural networks in recognizing patterns and features under geometric transformations.
	 Group convolutional neural networks enhance traditional convolutional neural networks by incorporating group-based geometric structures into their design.
	This research studies affine invariance on continuous-domain convolutional neural networks. 
	Despite other research considering isometric invariance or similarity invariance, we focus on the full structure of affine transforms generated by  the group of all invertible $2 \times 2$ real matrices (generalized linear group $\mathrm{GL}_2(\mathbb{R})$).
	We introduce a new criterion to assess the invariance of two signals under affine transformations. The input image is embedded into the affine Lie group $G_2 = \mathbb{R}^2 \ltimes \mathrm{GL}_2(\mathbb{R})$ to facilitate group convolution operations that respect affine invariance.
	Then, we  analyze the convolution of embedded signals  over $G_2$. In sum, our research could eventually extend the scope of geometrical transformations that usual deep-learning pipelines can handle.}




\maketitle



\section{Introduction}





Traditional neural networks are very important in practice and have demonstrated strong performance across a wide range of tasks \cite{abiodun2018state,widrow1994neural, akbari2022disaster, shahin2001artificial}. However, they face computational challenges when dealing with high-dimensional inputs and lack built-in mechanisms for translation invariance \cite{thompson2020computational, mohades2023cardinality}. Convolutional Neural Networks (CNNs) address these limitations by leveraging local connectivity and shared weights, making them well-suited for tasks like image recognition where spatial structure and translation invariance are essential \cite{o2015introduction}. CNNs have achieved remarkable success in analyzing, recognizing, and understanding images, primarily due to their ability to automatically extract useful features from raw data. However, the types of structures that CNNs can capture are often limited to simple symmetries, which may restrict their effectiveness in more complex pattern recognition scenarios. In ordinary CNNs, translation symmetry of signals can often be detected due to the inherent translation equivariance of convolutional layers.  
In Group-Convolutional Neural Networks (G-CNNs), this idea is extended to more general group symmetries, enabling better handling of more complex transformations and potentially improving learning efficiency.


In this paper, we aim to develop a gorup convolutional neural network  that categorizes two input signals as belonging to the same category, when these two signals can be converted into each other through an affine transformation. For example consider two input pictures, $f_1:\mathbb{R}^2 \to \mathbb{R}$ and $f_2:\mathbb{R}^2 \to \mathbb{R}$. A question is whether there is any affine transform that can transform the first picture into the second. This question is akin to investigating if $f_1(\boldsymbol{Ax+b}) = f_2(\boldsymbol{x})$, or at least if the distance between these two images is small.     Selecting a proper kernel can
 lead to stability of convolution under affine transform
(by stability under affine transform we mean $f( \boldsymbol{x})*g = f(\boldsymbol{Ax+b})*g$). 

However, as we will discuss later, it is shown that Equation (\ref{KernelConstraint}) does not have a solution in the general case and conventional convolutional neural network architectures are not useful in  in detecting general affine transforms. Therefore, we need to employ a different three layers convolutional neural network architecture. The idea is to embed the input image into $G_2$ (a process we call lifting), where we can perform convolutions in $G_2$, which is invariant under affine transformation.  This first layer is called the lifting layer. Then we need to perform the convolution in the convolution layer and finally, we have the projection layer which maps the signal in $G_2$ to ordinary signals.  Theorems \ref{1stlayerinvariance}, \ref{2ndlayerinvariance}, and \ref{3rdlayerinvariance} show that if distance of an input function and a suitable affine transform of another input function is less than $\epsilon$ for a fixed $\boldsymbol{A}$ and $\boldsymbol{b}$, are $\epsilon'$-affine invariant in the lifting layer, convolutional layer, and projection layer respectively. Therefore the the output of the designed convolutional neural network is $c\epsilon$-affine invariant (for a constant $c$). We show that this network design is stable for affine invariant inputs.  Furthermore, we demonstrate that the associated computations for these architectures can be simplified from complex convolutions over the transformation group to more straightforward integrals over real space.

This computational advantage becomes particularly meaningful in the context of Group Convolutional Neural Networks (G-CNNs), which extend CNNs to capitalize on the intrinsic geometric properties and symmetries in data, particularly images \citep{cohenc16}.
Unlike their traditional counterparts, G-CNNs harness the power of group theory, a mathematical framework that formalizes transformations and symmetries. This theoretical foundation ensures equivariance with respect to transformations described by the group, thereby enabling the network to maintain predictable behavior under various transformations. 

One striking characteristic of G-CNNs is their ability to characterize geometric features and symmetries throughout the network's architecture. Notably, they excel when dealing with large groups that extend beyond mere translation equivariance. Classical CNNs can be regarded as a special instance of G-CNNs. The real power of G-CNNs becomes evident when more intricate geometric transformations are at play.
Recent G-CNNs elevate feature maps to higher-dimensional, disentangled representations \citep{bekkers2019}. 
 Within these representations, G-CNNs effectively learn intrinsic geometric patterns and transformation behaviors present in the data—such as orientation, scale, and position—thereby reducing the need for traditional geometric data-augmentation techniques.
This not only streamlines the learning process but may also reduces the risk of overfitting.
 Moreover, G-CNNs maintain their predictive behavior under geometric transformations, thanks to their foundation in group theory and, therefore, give rise to the concept of equivariance.
The introduction of G-CNNs to the machine-learning community by \citep{cohenc16} marked the inception of an expanding body of G-CNN literature that consistently highlights many advantages of G-CNNs over conventional CNNs. This literature can be roughly classified into three main categories: discrete G-CNNs, regular continuous G-CNNs, and steerable continuous G-CNNs. Discrete G-CNNs delve into discrete group structures, yielding improved performance in various applications. This approach has been explored in studies by \citep{cohenc16,winkels20183d,dieleman2016exploiting,worrall2018cubenet,hoogeboom2018hexaconv}, collectively contributing to the foundational understanding and practical deployment of discrete G-CNNs.

Regular continuous G-CNNs, as investigated by \citep{oyallon2015deep,bekkers2015training,weiler20183d,zhou2017oriented}, focus on seamless transformations within continuous domains. Their research showcases how G-CNNs can excel in handling continuous data, offering advantages over traditional CNNs in capturing intricate patterns and representations.
Steerable continuous G-CNNs, explored \citep{cohen2018spherical,worrall2017harmonic,kondor2018generalization, thomas2018tensor,andrearczyk2019exploring}, introduce a specialized approach where the convolution kernels are represented in terms of circular or spherical harmonics. This technique, particularly suitable for unimodular groups like roto-translations, enables efficient computation by utilizing basis coefficients.  \citep{knigge2022exploiting} employs separability convolutions to attain equivariance concerning scale-rotation-translation transformations.  \citep{chen2021equivariant} utilizes separability for efficient implementations of $SE(3)$ equivariance.

Our research investigates the property of affine invariance in the context of continuous-domain convolutional neural networks. Our focus are affine spaces formed by the generalized linear group $\mathrm{GL}_2(\mathbb{R})$, the group of all invertible matrices of size $2 \times 2$.  Affine transformations are fundamental operations that combine linear transformations and translations. These transformations are important because they  address distortions of an affine nature. For example, such distortions arise in certain types of CAPTCHA  \citep{CAPTCHA1,AffineArticle1, AffineArticle2} (see Figure \ref{captcha}).  Previous attempts have been made to investigate spaces that maintain affine-equivariance, but they are restricted to strict conditions, such as cases where the determinant equals 1 (expressed as $\mathrm{SO}(n)$). We instead consider affine-invariant spaces across the entire spectrum of invertible matrices. We introduce an  approach, where we assess whether the convolution of the lifting of $f_1$ and $f_2$ to $G_2$ exhibits $G_2$ invariance for every kernel. In order to apply this criterion, an additional step is required, namely, the computation of convolutions over $G_2$. We Solve this technical challenge using QR-decomposition discussed in \citep{schindler1993iwasawa}. 

This paper makes two main technical contributions:
\begin{enumerate}
	\item  We show that G-CNN architectures are stable under affine transformations generated by the generalized linear group $\mathrm{GL}_2(\mathbb{R})$.  This includes a  large spectrum of transformations, such as roto-translations and scale-translations. Thus, this paper is the first that proves the suitability of G-CNN architectures in such a generality.
	
	
	\item We show that the related computations for these architectures can be simplified from complicated convolutions over the transformation group to much simpler integrals over the real space.


\end{enumerate}

\noindent 
More broadly speaking, we show that G-CNNs cater to a considerably broader  spectrum of transformations than what was established before. This is a general result, but it can also be used in specific scenarios to analyze invariant inputs in affinely generated transformations, such as the roto-translation transformation.


\begin{figure}[h]
	\centering
	\includegraphics[width=8cm,height=1.5cm]{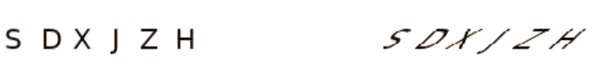}
	\caption{ Original letters and its  affine invariant CAPTCHA.}
	\label{captcha}
\end{figure}




\subsection{Preliminaries}

This section establishes the foundational concepts, terminology, and context along with some illustrative examples. These preliminaries set the stage for the main contributions and discussions presented in the following sections.

Ensuring the equivariance of artificial neural networks (NNs) with respect to a group $G$ is an essential characteristic, as it guarantees that applying transformations to the input preserves all information, merely shifting it to different network locations. It has been determined that when aiming for equivariant NNs, the sole viable choice is to employ layers in which the linear operator is defined through group convolutions. The journey to this conclusion commences with the conventional definition of neural networks layers given by

\begin{equation}\label{NNs}
	\boldsymbol{y}=\sigma\left(\mathcal{K}_{\boldsymbol{w}}f(\boldsymbol{x})+\boldsymbol{b}\right),
\end{equation}
where $\boldsymbol{x} \in \mathcal{X}$ represents the input vector, and $f$ denotes the input signal,  for example in an image input, $\boldsymbol{x}$ is the location of pixels and $f$ is the value that image takes in each pixel. Moreover,  $\mathcal{K}_{\boldsymbol{w}}: \mathcal{X} \rightarrow \mathcal{Y}$ stands as a linear map parameterized by the weight vector $\boldsymbol{w}$, with $\boldsymbol{b} \in \mathcal{Y}$ as a bias term, and $\sigma$ is the activation function. The kernel operator $\mathcal{K}_{\boldsymbol{w}}$ is also defined as follows

\begin{equation*}
	\mathcal{K}_{\boldsymbol{w}} f= \int_{\mathcal{X}} k(\boldsymbol{x},\boldsymbol{y}) f(\boldsymbol{x}) d\lambda_{\mathcal{X}}(\boldsymbol{x}),
\end{equation*}
where $\mathcal{K}_{\boldsymbol{w}} :  \mathbb{L}_2(\mathcal{X}) \to \mathbb{L}_2(\mathcal{Y})$, $d\lambda_{\mathcal{X}}$ is a Radon measure on $\mathcal{X}$, $k$ denotes the kernel function, and  $f \in \mathbb{L}_2(\mathcal{X})$ (square integrable function). To broaden the application of this explanation to the notion of group convolutional neural networks, we revisit a number of crucial definitions.



\begin{definition}[Group]
	A group $(G,\cdot)$ is a set $G$ equipped with a binary operator represented by a dot symbol. The dot operator is associative $((g_1 \cdot g_2) \cdot g_3 = g_1 \cdot (g_2 \cdot g_3)  )$, has an identity element $(e)$. Moreover, every element of the set has an inverse element $(g \cdot g^{-1} = g^{-1} \cdot g = e )$.
\end{definition}
In this context, the set comprises functions, such as translations or rotations. The group operation operates on elements of this set through addition or multiplication \citep{herstein1991topics}. We also need to define normal groups. A normal group is a subgroup $N$ of a group $G$ such that, for every element $g$ in $G$, the conjugate $gNg^{-1}$ is contained within $N$. 

\begin{example}[Translation group]
	The translation group
	in $\mathbb{R}^2$ is denoted by $(\mathbb{R}^2,\cdot)$ consists of all possible translations and is equipped with the below group product and group inverse:
	
	\begin{equation*}
		\begin{split}    
			&g\cdot g' = (\boldsymbol{x}+\boldsymbol{x}') \\
			&g^{-1} = -\boldsymbol{x},
		\end{split}
	\end{equation*}
	where $g = (\boldsymbol{x})$ and $g^{-1} = (-\boldsymbol{x})$ and $\boldsymbol{x} , \boldsymbol{x}' \in \mathbb{R}^2$.
\end{example}
One important example of groups are Lie groups, which are defined as follows:

\begin{definition}[Lie groups]
	A Lie group is a set $G$ with two structures: $G$ is a group and $G$ is a (smooth, real)
	manifold. These structures agree in the following sense: multiplication and inversion are smooth maps.
\end{definition}
Roto-translation symmetries of Euclidean spaces are examples of Lie groups, which is explained in the next example.

\begin{example}[Roto-translation group]
	
	The roto-translation group in $\mathbb{R}^2$ is denoted by $\mathrm{SE}(2)$. The group $\mathrm{SE}(2)=\R^2 \rtimes \mathrm{SO}(2)$ (where $\rtimes$ denotes semidirect product. In a direct product $G = H \times K$, both $H$ and $K$ are normal in $G$. Semidirect products are a relaxation of direct products where only one of the two subgroups must be normal)  consists of translations vectors in $\mathbb{R}^2$, and rotations in $\mathrm{SO}(2)$ and is equipped with the group product and group inverse:
	
	\begin{equation*}
		\begin{split}    
			&g.g' = (\boldsymbol{x},\boldsymbol{R}_{\theta})\cdot(\boldsymbol{x}',\boldsymbol{R}_{\theta'}) = (\boldsymbol{R}_{\theta}\boldsymbol{x}'+\boldsymbol{x},\boldsymbol{R}_{\theta+\theta'})\\
			&g^{-1} = (-\boldsymbol{R}_{\theta}^{-1}\boldsymbol{x},\boldsymbol{R}_{\theta}^{-1}),
		\end{split}   
	\end{equation*}
	for $g = (\boldsymbol{x},\boldsymbol{R}_{\theta})$, $ g'=(\boldsymbol{x}',\boldsymbol{R}_{\theta'})$, and $\boldsymbol{R}_{\theta} = \begin{pmatrix}
		\cos\theta & -\sin\theta \\
		\sin \theta & \cos\theta 
	\end{pmatrix}.$
\end{example}
The group operator provides instructions on how to act on the group elements, ensuring that the result remains within the group. Of particular interest are symmetry groups, where each element in the set represents a symmetry transformation. When the group acts on a specific space, it is referred to as a group action. 

\begin{definition}[Group action]
Let $\chi$ be a set. If $G$ is a group with identity element $e$, then a group action $\alpha$ of $G$ on $\chi$ is a function, $\alpha: G \times \chi \to \chi$, (which usually denotes as $\alpha(h, \boldsymbol{x}) = h \odot \boldsymbol{x}$)
that satisfies identity and compatibility conditions ($e \odot \boldsymbol{x} = \boldsymbol{x}$, $g\odot (h \odot \boldsymbol{x}) = (g\cdot h) \odot \boldsymbol{x}$) for all $g, h\in G$ and all $\boldsymbol{x} \in \chi$.
\end{definition}
For example the action of group $G = \mathrm{SO}(d)$ on space $\chi = \R^d$ could be denoted by $g \odot \boldsymbol{x} = \boldsymbol{R} \boldsymbol{x}$, where $\boldsymbol{x} \in \R^d$ and $\boldsymbol{R}  \in \mathrm{SO}(d)$. 
For the set of points, we perform transformation through group products. While in the convolution kernel, we perform transformation via group representations.  Therefore we need to understand representations. The multiplication within a group instructs us on merging transformations, yet it does not provide guidance on utilizing these transformations on other entities like vectors or signals. To address this, we require the concept of group action and group representations. Nevertheless, frequently, our attention is predominantly directed towards linear group actions operating on vector spaces, and these actions are termed representations.

\begin{definition}[Representation] 
A representation is an invertible linear transformation $\rho(g): V \rightarrow V$ parameterized by group elements $g,h \in G$ that acts on some vector space $V$, which follows the group structure (it is a group homomorphism) via
$$
\rho(g) \rho(h) v=\rho(g \cdot h) v
$$
for $v \in V$.

\end{definition}

\begin{definition}[Regular representation]
Let $f \in \mathbb{L}_2(\mathcal{X})$. Then the regular representation of $G$ acting on $\mathbb{L}_2(\mathcal{X})$ is given by
$$
\rho(g) f(\boldsymbol{x})=f\left(g^{-1}  \boldsymbol{x}\right) .
$$
\end{definition}



\begin{example}[Regular representation of roto-translation group]
Let $f \in \mathbb{L}_2 (\mathbb{R}^2)$  be a two dimensional image, $G = \mathrm{SE}(2)$ denotes the roto-translation group then 
\begin{equation*}
	\rho(g)f( \boldsymbol{y} ) =  f( \boldsymbol{R}_{\theta}^{-1}(\boldsymbol{y}  - \boldsymbol{x}) ).  
\end{equation*}

\end{example}
We continue this part with some additional definitions that we need in the next section.

\begin{definition}[Coset]
Let $H \subset G$ be a subgroup of $G$. Then $g H$ denotes a coset given by
\begin{equation*}
	g H=\bigl\{g \cdot h \mid h \in H \bigr\}.
\end{equation*}
\end{definition}

\begin{definition}[Quotient Space]
Let $H \subset G$ be a subgroup of $G$. Then $G / H$ denotes the quotient space that is defined as the collection of unique cosets $g H \subset G$. Elements of $G / H$ are thus cosets that represents an equivalence class of transformations for which $g \sim \tilde{g}$ are equivalent if there exists a $h \in H$ such that $g=\tilde{g} h$.
\end{definition}

\begin{definition}[Stabilizer]
Let $G$ acts on $\mathcal{X}$ via the action $\odot$. For every $\boldsymbol{x} \in \mathcal{X}$, the stabilizer subgroup of $G$ with respect to the point $\boldsymbol{x}$ is denoted with $\operatorname{Stab}_G(\boldsymbol{x})$ is the set of all elements in $G$ that fix $\boldsymbol{x}$
\begin{equation*}
	\operatorname{Stab}_G(\boldsymbol{x})=\bigl\{g \in G \mid g \odot \boldsymbol{x}=\boldsymbol{x}\bigr\}.
\end{equation*}
\end{definition}
Moreover from \citep{bekkers2019} we know that,  if $\mathcal{X}$ be a homogeneous space of $G$. Then $\mathcal{X}$ can be identified with $G / H$ with $H=\operatorname{Stab}_G\left(\boldsymbol{x}_0\right)$ for any $\boldsymbol{x}_0 \in \mathcal{X}$.

In order to study resemblance of two input functions $f_1$ and $f_2$ we need to provide some definitions.  

\begin{definition}[Affine invariance]
We say that functions $f_1,f_2 \in \mathbb{L}(\R^2)$ are  $\epsilon$-equivalance if  $\|f_1-f_2 \|_1 < \epsilon$ or $\sup _{\boldsymbol{x}} |f_1(\boldsymbol{x})-f_2(\boldsymbol{x}) | < \epsilon$.
\end{definition}

For instance, if a picture shows a slight deviation due to noise, we aim to overlook or disregard this deviation. 

\begin{definition}[Affine invariance]
We say that functions $f_1,f_2 \in \mathbb{L}(\R^2)$ are Affine invariant if there exists $\boldsymbol{A} \in G_2$ so that  $f_1=\rho(\boldsymbol{A})f_2$.
\end{definition}


\begin{definition}[$\epsilon$-Affine invariance]
We say that functions $f_1,f_2 \in \mathbb{L}(\R^2)$ are $\epsilon$-Affine invariant if there exists $\boldsymbol{A} \in G_2$ so that $\|f_1-\rho(\boldsymbol{A})f_2 \|_1 < \epsilon$ or $\sup _{\boldsymbol{x}} |f_1(\boldsymbol{x})-\rho(\boldsymbol{A})f_2(\boldsymbol{x}) | < \epsilon$.

\end{definition}

For simplicity in notation we do not use $\cdot$ and $\odot$ symbols in the next sections. Furthermore, in this paper we use $g$ and $h$ to denote group elements and $f$ and $k$ to denote functions. 

\subsection{Group convolutional neural networks architecture}

One conventional method to build group convolutional neural networks is to apply isotropic convolutions for Equation (\ref{NNs}).
An isotropic $\mathbb{R}^d$ convolution layer maps between planar signals $\mathbb{L}_2\left(\mathbb{R}^d\right)$  with $\mathcal{K}$ a planar correlation given by
\begin{equation*}
(\mathcal{K} f)(\boldsymbol{y})=\int_{\mathbb{R}^d} k(\boldsymbol{x}-\boldsymbol{y}) f(\boldsymbol{x}) \mathrm{d} \boldsymbol{x},
\end{equation*}
and in which it is shown (Theorem 1 from \cite{bekkers2019}) that $k$ satisfies
\begin{equation}\label{KernelConstraint}
\mathrm{for \ all} \ h \in H: \quad k(\boldsymbol{x})=\frac{1}{|\operatorname{det} h|} k\left(h^{-1} \boldsymbol{x}\right).
\end{equation}




Applying isotropic convolutions is limiting because they are constrained by the shape of the kernels. One approach to overcome this limitation, is to lift the signals to the group $G$. Lifting of the input signal, not only addresses the constraints of kernels as noted by \citep{bekkers2019} but also offers advantages in enhancing the performance of image processing, as highlighted in the work by \citep{smets2023pde}.
When we apply lifting we must look for stabilizer $\operatorname{Stab}_G$ when $G$ acts on $G$. In this case we have 

\begin{equation*}
H = \operatorname{Stab}_G(g) = \bigl\{x \in G | xg = g\bigr\} = e.  
\end{equation*}
As a result, Equation (\ref{KernelConstraint}) is fulfilled for all kernels, and there are no longer any limitations imposed on the choice of kernels.

\begin{definition}[Lifting layer $(\mathcal{X}=\mathbb{R}^d, \mathcal{Y}=G)$]
Let $g = (\boldsymbol{x},\boldsymbol{P})$, where $\boldsymbol{x} \in \R^2$ and $\boldsymbol{P} \in \mathrm{GL}_2(\mathbb{R})$. Also let $k: \R^2 \to \R$ be a compact supported distribution.    
A lifting layer maps from $\mathbb{L}_2\left(\mathbb{R}^d\right)$ to $\mathbb{L}_2(G)$ on the group $G$. 
A lifting correlation is given by
\begin{equation*}
	(\mathcal{K} f)(g)=\int_{\mathbb{R}^d} \frac{1}{|\operatorname{det} \boldsymbol{P}|} k\left(g^{-1} \tilde{\boldsymbol{x}}\right) f(\tilde{\boldsymbol{x}}) \mathrm{d} \tilde{\boldsymbol{x}}.
\end{equation*}
\end{definition}

\begin{example}[Lifting for Kronecker delta kernel]\label{LiftingExample}
Let 
\begin{equation*}
	k=\delta(\boldsymbol{x},\boldsymbol{0}_{d\times d}) =
	\begin{cases}
		1 & \text{if $\boldsymbol{x}=\boldsymbol{0} \in \R^d$};  \\
		0 & \text{otherwise.}
	\end{cases}
\end{equation*}
Then for $g = (\boldsymbol{x} , h) $ we have 
\begin{equation*}
	g^{-1}\tilde{\boldsymbol{x}} =  \boldsymbol{P}^{-1}(\tilde{\boldsymbol{x}} - \boldsymbol{x}). 
\end{equation*}
Therefore, 

\begin{equation*}
	k(g^{-1}\tilde{\boldsymbol{x}}) = \delta(\boldsymbol{P}^{-1}(\tilde{\boldsymbol{x}} - \boldsymbol{x}),\boldsymbol{0}_{d\times d}) =
	\begin{cases}
		1 & \text{if $\tilde{\boldsymbol{x}}=\boldsymbol{x}$};  \\
		0 & \text{otherwise.}
	\end{cases}
\end{equation*}
This results that the lifting layer is as the below
\begin{equation*}
	(\mathcal{K} f)(g) = \frac{f(\boldsymbol{x})}{|\det \boldsymbol{P}|}.
\end{equation*}
\end{example}

We also need to discuss this fact that the lifting layer integral exists. A function $f$ on $\mathbb{R}$ is called locally integrable if $f$ is integrable on every bounded interval $[a, b]$ for $a<b$ in $\mathbb{R}$. If $k \in C_c^{\infty}(\mathbb{R})$ and $f$ is locally integrable, then
\begin{equation*}
(f * k)(y)=\int_{-\infty}^{\infty} f(t) k(y-t) d t,
\end{equation*}
exists and is infinitely differentiable on $\mathbb{R}$. 
First of all the input $f$ is usually a picture and therefore the function $f$ is bounded. On the other hand the value of lifted functions on cosets is equal to that of $f$. Therefore the lifted function is bounded as well. We further know that the kernel is locally supported, which results the integrability. After lifting layer we will apply convolution layer which is defined as follows.

\begin{definition}[Group convolution layer $(\mathcal{X}=\mathcal{Y}=G)$]

A group convolution layer maps between $G$-feature maps in $\mathbb{L}_2(G)$. A group convolution is given by
\begin{equation*}
	(f*k)(h)=\int _{G}f(h)k\left(h^{-1}g\right)\,d\mu_G,
\end{equation*}
where $g \in G$ and $\mu_G$ is a Haar measure.
\end{definition}
We also need another layer to again maps to feature maps in $\mathbb{L}_2(\R^d)$, which can be used to imply smoothness of the output. 
\begin{definition}[$\R^d$ Projection layer]
A projection layer maps
between $G$-feature maps in $\mathbb{L}_2(G)$ back to planar feature maps in $\mathbb{L}_2(\R^d)$
\begin{equation}\label{projectionlayer}
	(\mathcal{K}f)(\boldsymbol{x}) = \int_{\tilde{H}} f(\boldsymbol{x},\tilde{h}) d{\tilde h}.
\end{equation}
\end{definition}

Frequently, the focus is on constructing architectures that are invariant, as opposed to equivariant. Invariance to all transformations in \(G\) is accomplished through mean pooling across the entire group \(G\), akin to how global translation invariance is typically obtained by mean or max pooling over the spatial dimensions of feature maps. The global pooling layer can be defined as follows.

\begin{definition}[Global pooling layer]
	A global pooling layer transforms any feature map  into a single scalar value. It adheres to the standard form of the layer as following, 
	
	\begin{equation}
		(\mathcal{K} f) = \int_X f(x) d\mu(x),
	\end{equation}
	where \(\mathcal{K}\) represents a pooling operation over \(X\) and $d\mu(x)$ is a Radon measure on $X$.
\end{definition}

\section{Main Result}
This section reviews the main result of this paper. In the way we explore affine invariant spaces and investigate the convolution integration over $G_2$.

\subsection{Problem Statement}

Our goal is to     
study invariance in affine transformations in continuous-domain convolutional neural networks.  An affine transformation basically combines linear transformations and translations. Affine transformations are denoted as follows

\begin{equation*} 
G_2 = \Bigl\{[\boldsymbol{x},\boldsymbol{A}] : \boldsymbol{x} \in \R^2, \boldsymbol{A} \in \mathrm{GL}_2(\R)  \Bigr\},
\end{equation*} 
where 
\begin{equation*} 
[\boldsymbol{x},\boldsymbol{A}] : \boldsymbol{z} \mapsto \boldsymbol{x} + \boldsymbol{A}\boldsymbol{z}.
\end{equation*} 
The identity is $[\boldsymbol{0},\boldsymbol{I}]$, and, therefore, for all $\boldsymbol{B} \in \mathrm{GL}_2(\R)$ we have $[\boldsymbol{y},\boldsymbol{B}]^{-1} = [-\boldsymbol{B}^{-1}\boldsymbol{y},\boldsymbol{B}^{-1}]$.




The affine transformation is important as we may face affine  type distortions due proximity of the camera with respect to the object. For example, this type of affine distortion could manifest in remote sensing images, as well as in camera imagery which can include various perspective distortions \citep{AffineArticle1}. It is important to note that in an affine transformation, parallel lines in the original image continue to remain parallel in the transformed image.  However, the transformation can introduce distortion in the angles between lines.

		
		
		
		
		
		
		
		
		


This paper explores the use of convolutional neural networks in handling affine transformations, focusing specifically on cases where the transformation matrix $\boldsymbol{A}$
belongs to the general linear group
$\mathrm{GL}_2(\mathbb{R})$. We diverge from the use of isometric convolutions, opting instead for the application of the lifting-projection method, which we elucidate comprehensively. 
While prior investigations have focused on compact groups such as $\mathrm{SO}(2)$, it is important to highlight that the $\mathrm{GL}_2(\mathbb{R})$ group does not fall under the category of compact groups. 
Our alternative method focuses on analyzing the convolution of the lifted forms of the signals $f_1$ and $f_2$ for achieving $G_2$ invariance. We also need to introduce an extra step that encompasses performing convolutions on $G_2$ and address the unique challenges associated with this, including techniques for handling integrations over $G_2$. Therefore, the initial step is to demonstrate that our three-layer group-convolutional neural networks are stable under affine transformations produced by the generalized linear group $\mathrm{GL}_2(\mathbb{R})$. The below theorem addresses this.
\begin{theorem}[Stability of G-CNN architecture]\label{mainTheorem1}
if $\Sigma$ be the G-CNN consisting of
of three, lifting, convolutional, and $\mathbb{R}$-projection layers and if the distance of a function $f_1$ and the affine transform of another function $f_2$ be less than $\epsilon$, then $|\Sigma f_1 - \Sigma f_2| < c \epsilon$. Where $c = \|k_1\|_1^{\R^2} \|k_2\|_1^{G_2}$ and $k_1$ and $k_2$ are kernels of lifting layer and convolution layer respectively.
\end{theorem}
\begin{proof}
We analyze the three layers of the network individually, demonstrating the stability of each layer. This, in turn, ensures the stability of the entire network. Theorems \ref{1stlayerinvariance}, \ref{2ndlayerinvariance}, and \ref{3rdlayerinvariance} detail this process. The current theorem is a direct consequence of these theorems.
\end{proof}

\begin{remark}
	Example \ref{LiftingExample} shows that he stability of affine-invariant systems can decrease as the determinant of the transformation matrix approaches zero. This aligns with geometric intuition: as $\det(\boldsymbol{P}) \to 0$,  the transformation projects initial objects up to a very small distance to lower-dimensional subspace. Consequently, inversion or reconstruction becomes more challenging.
	
\end{remark}

The next theorem asserts that, lifting layer does not change the affine invariance of input signals. The proof of the following theorems could be found in Appendix.

\begin{theorem}[Invariance in the lifiting layer]\label{1stlayerinvariance}
Let  $f_1, f_2: \R^2 \to \R$ are input signals and let there exists an $g \in G_2$ so that $\sup | (f_1 - \rho(g^{-1}) f_2| < \epsilon$ then   
\begin{equation*}
	\sup| (\mathcal{K}f_1)(g) - \rho(g^{-1}) (\mathcal{K}f_2)(g)| < \epsilon \|k\|_1^{\R^2}.
\end{equation*}
\end{theorem}

The next step is to demonstrate the invariance in the convolution layer.

\begin{theorem}[Invariance in the convolutional layer]\label{2ndlayerinvariance}
Let $(\mathcal{K}f_1),(\mathcal{K}f_2): G_2 \to \R$ be the lifting of $f_1,f_2: \R^2 \to \R$ and let there exists an $\tilde{h} \in G_2$ so that $\| (\mathcal{K}f_1)(g) - \rho(\tilde{h}) (\mathcal{K}f_2)(g)\|_{\sup} < \epsilon$ then 
\begin{equation*}
	\| (\mathcal{K}f_1)*k - \rho(\tilde{h}) (\mathcal{K}f_2)*k\|_{\sup}^{G_2} < \epsilon \|k\|_1^{G_2}
\end{equation*}
holds for every kernel $k$ and vice-versa. Where $\| f\|_1^{G_2} = \int_{G_2}|f| d\mu_{G_2}$.
\end{theorem}


The aforementioned finding indicates that to assess the equivalence of two signals, it is necessary to perform a convolutional integration across $G_2$. We investigate this problem in the next section. In the last step we provide the below theorem which states that the function  
$c_{(\mathcal{K},k)} : C(\R^2 , \R) \to \R $ 
defined by 
$\int_{G_2} (\mathcal{K}f)*k \, d\mu_{G_2}(g) $ can be used for characterization of invariant affine functions in the projection layer.

\begin{theorem}[Invariance in the projection layer]\label{3rdlayerinvariance}
If $(\mathcal{K}f_1),(\mathcal{K}f_2) : G_2\to \R$ are lifting of input signals and there exists a $\tilde{h} \in G_2 $ such that $\|(\mathcal{K}f_1)-\rho(\tilde{h})(\mathcal{K}f_2)\|_1^{G_2} < \epsilon$. Then we have 
\begin{equation*}
	\Bigr|\int_{G_2} \bigr((\mathcal{K}f_1)*k -(\mathcal{K}f_2)*k\bigl)(h) d\mu_{G_2}(h) \Bigl|   \le \epsilon \|k\|_1^{G_2}
\end{equation*}
\end{theorem}

\subsection{Convolution Computation}

Before illustrating how to compute the convolution over the group $G_2$, we remark some ingredients which is essential to compute the convolution over $G_2$. We finally show that the convolution over $G_2$ can be computed through Fourier transform and  integration over real valued  space. 

In our study, we adopt a straightforward approach to calculate the $G_2$-invariant convolution for a broader kernel, which can be formulated as follows:


\begin{equation}
\int_{G_2} f([\boldsymbol{x},\boldsymbol{A}])k([\boldsymbol{y},\boldsymbol{B}]^{-1}[\boldsymbol{x},\boldsymbol{A}]) d \mu_{G_2}. 
\end{equation}
Using the Stone–Weierstrass theorem, in the setup of continuous functions with respect to sup-norm, $C({G_2},\R)=C(\mathrm{GL}_2(\R)\ltimes \R^2 , \R)$,  which asserts that summation of separable functions are dense in $C({G_2},\R)$, we reduce the kernel sets to functions of form $k(\boldsymbol{y},\boldsymbol{A}) = \sum_{i=1}^{M} k_{1_i}(\boldsymbol{y})k_{2_i}(\boldsymbol{A}). $
This reduction help us to benefit Fourier transforms to simplify some parts of our calculations. We use QR parametrization of $\mathrm{GL}_2(\R)$ which aids us in utilizing numerical approaches,  for example are introduced in \citep{eshkuvatov2013polynomial}.
Now, we illustrate the outcomes presented in \citep{schindler1993iwasawa,milad2023harmonic}, which are pertinent to our calculations.
Let 
\begin{equation*}
K_0=\left\{\left(\begin{array}{cc}
	s & -t \\
	t & s
\end{array}\right): s, t \in \mathbb{R}, s^2+t^2>0\right\}, 
\end{equation*}
and 
\begin{equation*}
H_{(1,0)}=\left\{\left(\begin{array}{ll}
	1 & 0 \\
	u & v
\end{array}\right): u, v \in \mathbb{R}, v \neq 0\right\}.
\end{equation*}
It is shown that $\mathrm{GL}_2(\mathbb{R})=K_0 H_{(1,0)}, K_0 \cap H_{(1,0)}=\boldsymbol{I}$, where $\boldsymbol{I}$ denotes the identity matrix, and $(\boldsymbol{M}, \boldsymbol{C}) \rightarrow \boldsymbol{M C}$ is a homeomorphism of $K_0 \times H_{(1,0)}$ with $\mathrm{GL}_2(\mathbb{R})$.  From \citep{milad2023harmonic} we have
\begin{equation}\label{Eq:GnDecomposition}
\int_{G_n} f \, d \mu_{G_n}=\int_{\mathrm{GL}_n(\mathbb{R})} \int_{\mathbb{R}^n} f[\boldsymbol{x}, \boldsymbol{A}] \frac{d \boldsymbol{x} d \mu_{\mathrm{GL}_n(\mathbb{R})}(\boldsymbol{A})}{|\operatorname{det}(\boldsymbol{A})|} \text {, for all } f \in C_c\left(G_n\right),
\end{equation}
where $C_c(G)$ denotes the space of continuous $\mathbb{C}$-valued functions of compact support on $G$.
For any integrable function $f$ on $\mathrm{GL}_2(\mathbb{R})$, the Haar integral on $\mathrm{GL}_2(\mathbb{R})$ can be expressed as

\begin{equation}\label{GL_Decomposition}
\int_{\mathrm{GL}_2(\mathbb{R})} f \, d \mu_{\mathrm{GL}_2(\mathbb{R})}=\int_{K_0} \int_{H_{(1,0)}} f(\boldsymbol{M C})|\operatorname{det}(\boldsymbol{C})| d \mu_{H_{(1,0)}} d \mu_{K_0}.
\end{equation}
The map $[u, v] \rightarrow\left(\begin{array}{ll}1 & 0 \\ u & v\end{array}\right)$ is an isomorphism of the group $G_1=\mathbb{R} \rtimes \mathbb{R}^*$ with $H_{(1,0)}$.
When $n=1, \mathrm{GL}_1(\mathbb{R})$ can be identified with $\mathbb{R}^*$ and $G_1$ identified with $\mathbb{R} \rtimes \mathbb{R}$. We recall that $\int_{\mathbb{R}^*} f d \mu_{\mathbb{R}^*}=\int_{\mathbb{R}} f(b) \frac{d b}{|b|}$, where the integral on the right hand side is the Lebesgue integral on $\mathbb{R}$, and

\begin{equation}\label{Eq:G1Int}
\int_{G_1} f \, d \mu_{G_1}=\int_{\mathbb{R}} \int_{\mathbb{R}} f[y, b] \frac{d y d b}{b^2}.
\end{equation}

\subsection{Integral over $G_2$}
A difficult aspect in the implementation of group convolutional neural networks involves performing convolutions across the group. This segment addresses this particular challenge by delving into the problem, which we will break down into the more manageable tasks of calculating Fourier transforms and conducting integrations in real-valued space.
We have the below theorem for the integration over $G_2$.

\begin{theorem}\label{G2integration}
Let $\boldsymbol{A}=\left(\begin{array}{ll}a & b \\ c & d\end{array}\right) \in \mathrm{GL}_2(\mathbb{R})$ and let the kernel is separable meaning that $k(\boldsymbol{x},\boldsymbol{A}) = k_1(\boldsymbol{x})k_2(\boldsymbol{A})$ and consider the one to one transform between $H$ and $H^*$ so that $H^*(s,t,u,v,\boldsymbol{B},\boldsymbol{y}) := H_{f,k}(a,b,c,d,\boldsymbol{B},\boldsymbol{y})$, where $a = s-ut$, $c =t+us $, $b =-t/v $, and $d = s/v $, then we have

\begin{equation*}
	\int_{G_2} f([\boldsymbol{x},\boldsymbol{A}])k([\boldsymbol{y},\boldsymbol{B}]^{-1}[\boldsymbol{x},\boldsymbol{A}]) d \mu_{G_2} =  \int_{\R}  \int_{\R} \int_{\R} \int_{\R} H^*(s,t,u,v,\boldsymbol{B},\boldsymbol{y})\frac{du dv}{|v| }\frac{ds dt}{s^2+t^2}.
\end{equation*}
where 
\begin{equation*}
	H_{f,k}(\boldsymbol{A},\boldsymbol{B},\boldsymbol{y})  =   \frac{k_2(\boldsymbol{A}\boldsymbol{B}^{-1}) }{{|\operatorname{det}(\boldsymbol{A})||\operatorname{det}(\boldsymbol{B}^{-1})|}} \mathcal{F}^{-1}\Bigl(F(\boldsymbol{u}) K_1(\boldsymbol{B}^{\top}\boldsymbol{u}  ) \Bigr).
\end{equation*}
\end{theorem}

The proof of this theorem is discussed in the appendix. Applying this result we can use the numerical methods in \citep{eshkuvatov2013polynomial} to compute the former integral as it has singularity in $s=0, t=0$.
Note that we can write $K_0$ as $\R^+ \rtimes \mathrm{SO}(1)$ where
\begin{equation*}
\left(\begin{array}{cc}
	s & -t \\
	t & s
\end{array}\right) = (s^2+t^2) \times     \left(\begin{array}{cc}
	r \cos \theta & -r \sin \theta \\
	r \sin \theta & r \cos \theta
\end{array}\right).
\end{equation*}

The final step  that necessitates computation is the integration within the projection layer. In the context of our affine transformation, the stabilizer is specifically 
$\mathrm{GL}_2(\R)$. We refrain from delving into the intricacies of this process, as it bears resemblance to the earlier scenario.

Now we first provide one example on $G_1$ this example gives us some insights for examples on $G_2$. We also will see how two class of complicated functions are equivalent with some intervals on $\R$ which are easier to investigate. 
\begin{example}\label{equivalent_inputs}
Let the continuous input function be defined as $f(a)=\begin{cases}
	1 & \text{if $a \in [t_1,t_2]$} \\
	0 & \text{otherwise}
\end{cases},$
		then according to Example (\ref{LiftingExample}) for the lifting of $f$ we have $(\mathcal{K} f)[a,b]=\begin{cases}
			\frac{1}{|b|} & \text{if $a \in [t_1,t_2]$} \\
			0 & \text{otherwise}
		\end{cases},$
			then we have:
			\begin{equation*}
				\begin{split}
					&\int_{\mathbb{R}^2} \mathcal{K}(f[a,b])k([x,y]^{-1}[a,b])d\mu_{G_1}(a,b)  \\
					&=\int_{\mathbb{R}^2} \frac{1}{|b|}k([-\frac{y}{x},\frac{1}{x}][a,b])\frac{da db}{b^2}\\
					&=\int_{\mathbb{R}^2} \frac{1}{|b|b^2}k([\frac{a-y}{x},\frac{b}{x}])da db\\
				\end{split} 
			\end{equation*}
			We can define a separable kernel so that we make the computations easier. We define it as  $
			k(s,t) = s^3 \exp{(s)}\frac{1}{\sqrt{2\pi}}exp{(-t^2)}
			$. Employing this definition we have:
			\begin{equation*}
				\begin{split}
					&\int_{\mathbb{R}^2} \frac{1}{|b|b^2}k([\frac{a-y}{x},\frac{b}{x}])da db
					= \int_{t_1}^{t_2}\exp(\frac{a-y}{x}) \int_{\mathbb{R}}\frac{1}{|b|b^2}\frac{b^3}{x^3}\frac{1}{\sqrt{2\pi}}\exp(-\frac{b}{x})^2da db \\
					&= \int_{t_1}^{t_2}\exp(\frac{a-y}{x}) \int_{\mathbb{R}}\frac{1}{|b|b^2}\frac{b^3}{x^3}\frac{1}{\sqrt{2\pi}}\exp(-\frac{b}{x})^2da db
					= \frac{1}{x}\exp{(\frac{t_2-y}{x})}-\frac{1}{x}\exp{(\frac{t_1-y}{x})}
				\end{split} 
			\end{equation*}


		\end{example}
		
		\begin{figure}[h]
			\centering
			\includegraphics[width=9cm,height=7cm]{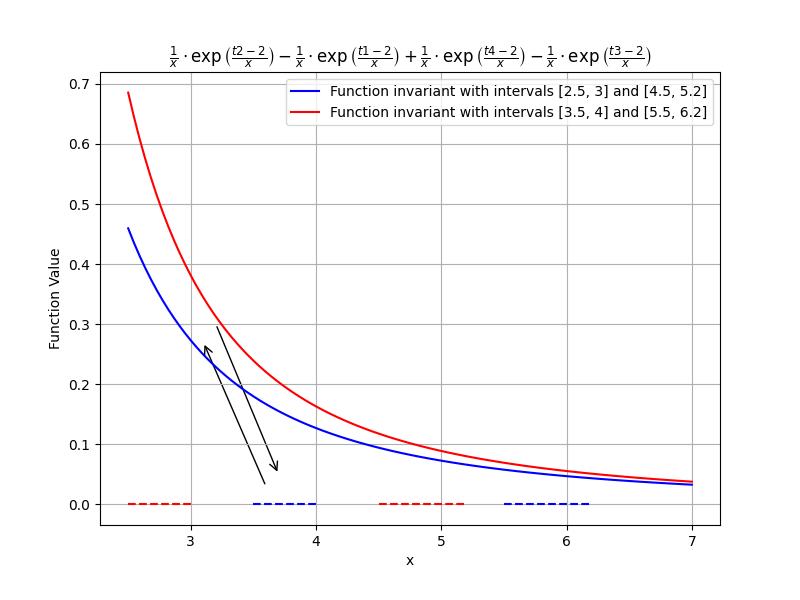}
			\caption{Two family of related affine invariant functions denoted by dashed and solid lines.}
			\label{captcha}
		\end{figure}

		This example establishes a connection between the invariance under affine transformations of more complex functions and the invariance of certain simple step functions. It also underscores the importance of selecting an appropriate kernel function.
		In the figure below, we selected two distinct intervals as input functions and then applied an affine transform 
		[1,1]. We then drew the equivalence functions for a fixed value of
		$y$. It shows that we can study invariance under affine transformations of more complex shapes through the invariance of certain simple intervals.

		\section{Experiments}
		
	In this section, we present simulation results to demonstrate the practical significance of our theoretical framework. Specifically, we constructed a neural network consisting of a three-layer group convolution module, followed by a fully connected neural network. The design of the group convolution layers was guided by the methodologies described in Examples \ref{LiftingExample} and \ref{equivalent_inputs}, which were used to define the kernels for the lifting and convolution operations. Projection is accomplished using straightforward averaging. For comparison, we also implemented a standard convolutional neural network that shares an identical fully connected component and uses the same set of parameters. Both models were evaluated on a dataset generated using affine transformations to test their performance under geometric variability. As the results will show, our proposed group convolution network can outperform the conventional CNN, particularly in scenarios where the number of input samples is relatively low. In the following, we provide a more detailed explanation of our simulation setup and results.

	To evaluate the proposed G-CNN, we generated a synthetic digit dataset using affine-transformed images of digits \(0\) through \(9\). Each digit was rendered using a standard font on a fixed-size grayscale canvas, followed by an affine transformation incorporating stretch, shear, rotation, and translation.
		A fixed number of transformed samples were created per digit class, with translation parameters sampled uniformly to introduce spatial variability. The final dataset contains \(10 \times N\) labeled images, scalable via \(N\), allowing for controlled experiments across models such as G-CNN and CNN Networks. 
	
	To evaluate the effectiveness of the proposed G-CNN, we compare its performance against a standard CNN. The experiment is conducted on a synthetic digit dataset, where digits $0$ through $9$ are rendered and then subjected to an affine transformation characterized by a fixed matrix, with a random bias added to distribute them within the frame. In the first experiment, we used the following matrix to define our affine transformation $A = \begin{bmatrix}
		2.5 & 0.7 \\
		0.6 & 1.8
	\end{bmatrix}.$
	
 As shown in Fig.~\ref{fig:performance_comparison}, the G-CNN  outperforms the standard CNN across all dataset sizes. The G-CNN demonstrates strong generalization even with limited data, achieving high accuracy with a few samples per class. In contrast, the standard CNN requires more data to reach comparable performance. We also trained our network on data generated using the matrix \(A = \begin{bmatrix} 1 & 0.7 \\ 0.7 & 1 \end{bmatrix}\) As illustrated in Fig.~\ref{fig:performance_comparison2}, the G-CNN outperform CNN when the input data size is limited.

		\begin{figure}[htbp]
			\centering
			\includegraphics[width=0.7\linewidth]{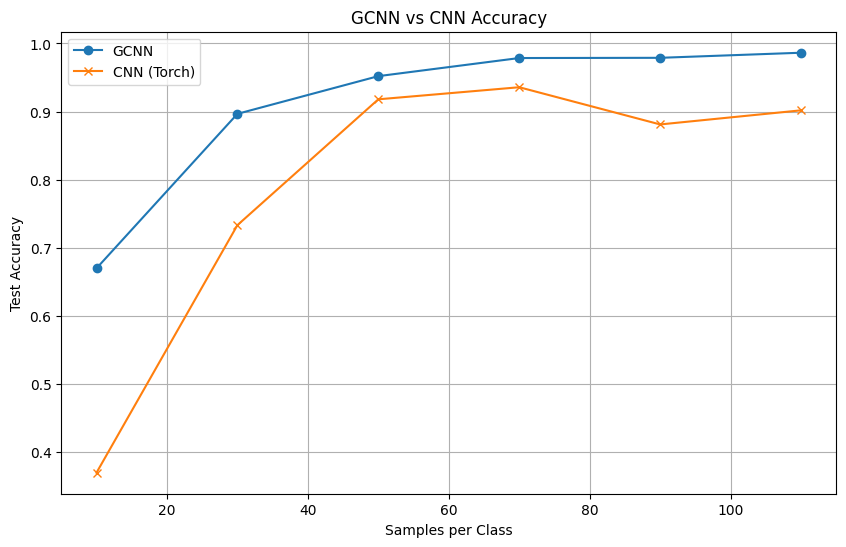} 
			\caption{Accuracy comparison between G-CNN and standard CNN across varying sample sizes. G-CNN outperforms CNN under affine transformation \(A = \begin{bmatrix} 2.5 & 0.7 \\ 0.6 & 1.8 \end{bmatrix}\).
			}
			\label{fig:performance_comparison}
		\end{figure}

	\begin{figure}[htbp]
	\centering
	\includegraphics[width=0.7\linewidth]{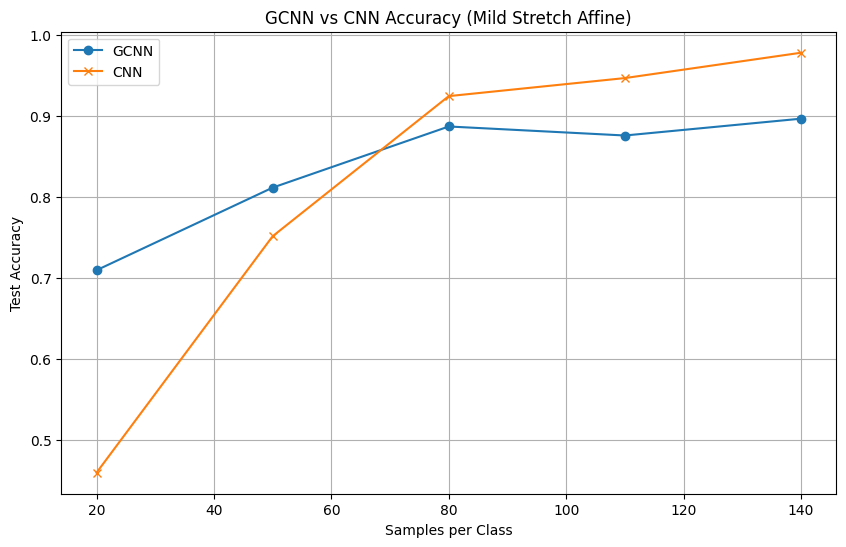} 
	\caption{Accuracy comparison between G-CNN and standard CNN across varying sample sizes. G-CNN outperforms CNN for relatively small sample sizes under affine transformation \(A = \begin{bmatrix} 1 & 0.7 \\ 0.7 & 1 \end{bmatrix}\).
	}
	\label{fig:performance_comparison2}
\end{figure}

To further assess the robustness of the proposed Group Convolutional Neural Network (G-CNN), we present a qualitative comparison with a standard convolutional neural network (CNN) on a transformed digit recognition task. The evaluation uses a synthetic dataset with 40 training samples per digit, where each digit is subjected to an affine transformation defined by $A = \begin{bmatrix}
	1 & 2 \\
	2 & 1
\end{bmatrix}$.


Affine transformations are composed of rotation, stretch, and translation, introducing significant geometric variability into the digit appearance. Fig.~\ref{fig:qualitative_comparison} shows the prediction results for both models on a set of test images, with each sample annotated by the true label (T) and predicted label (P).

The G-CNN demonstrates strong affine invariance, correctly identifying more digits under transformation compared to the standard CNN. Quantitatively, the G-CNN achieves a test accuracy of 0.6950, significantly outperforming the CNN, which reaches only 0.3150. These results highlight the advantage of incorporating group-equivariant structure in neural network design, particularly in scenarios involving geometric variability. 

We also compare the prediction performance of the G-CNN and a standard CNN under moderate affine transformations defined by $A = \begin{bmatrix} 1 & 0.7 \\ 0.7 & 1 \end{bmatrix}$.


As shown in Fig.~\ref{fig:affine_eval2}, the G-CNN demonstrates stronger robustness to geometric deformation, producing more accurate predictions across various transformed digit inputs. The G-CNN achieves a mean test accuracy of 0.800, outperforming the standard CNN, which reaches a mean accuracy of 0.720.

		\begin{figure}[htbp]
			\centering
			\begin{subfigure}[b]{0.47\textwidth}
				\centering
				\includegraphics[width=\linewidth]{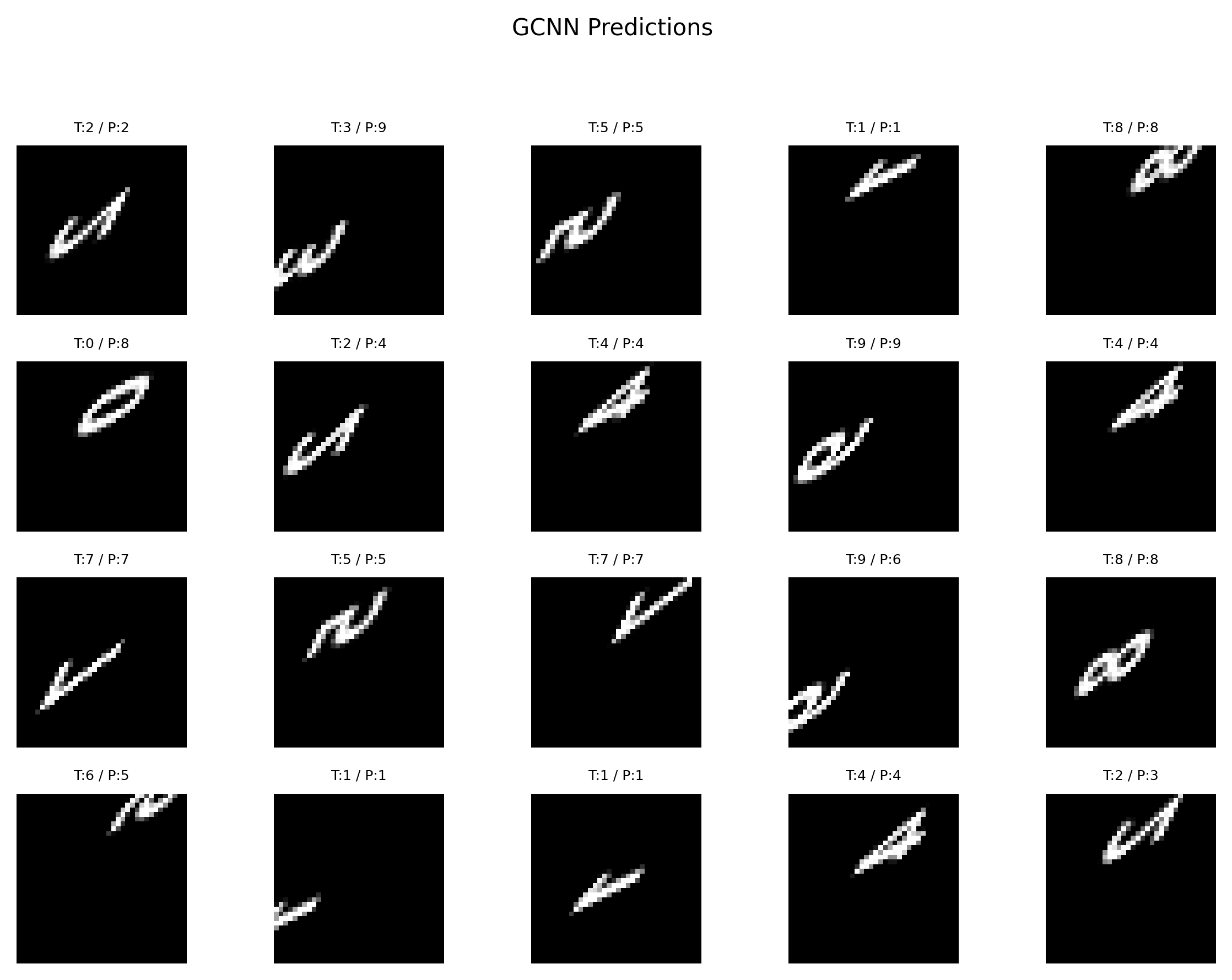}
				\caption{Group Convolution Network}
				\label{fig:group_conv}
			\end{subfigure}
			\hfill
			\begin{subfigure}[b]{0.47\textwidth}
				\centering
				\includegraphics[width=\linewidth]{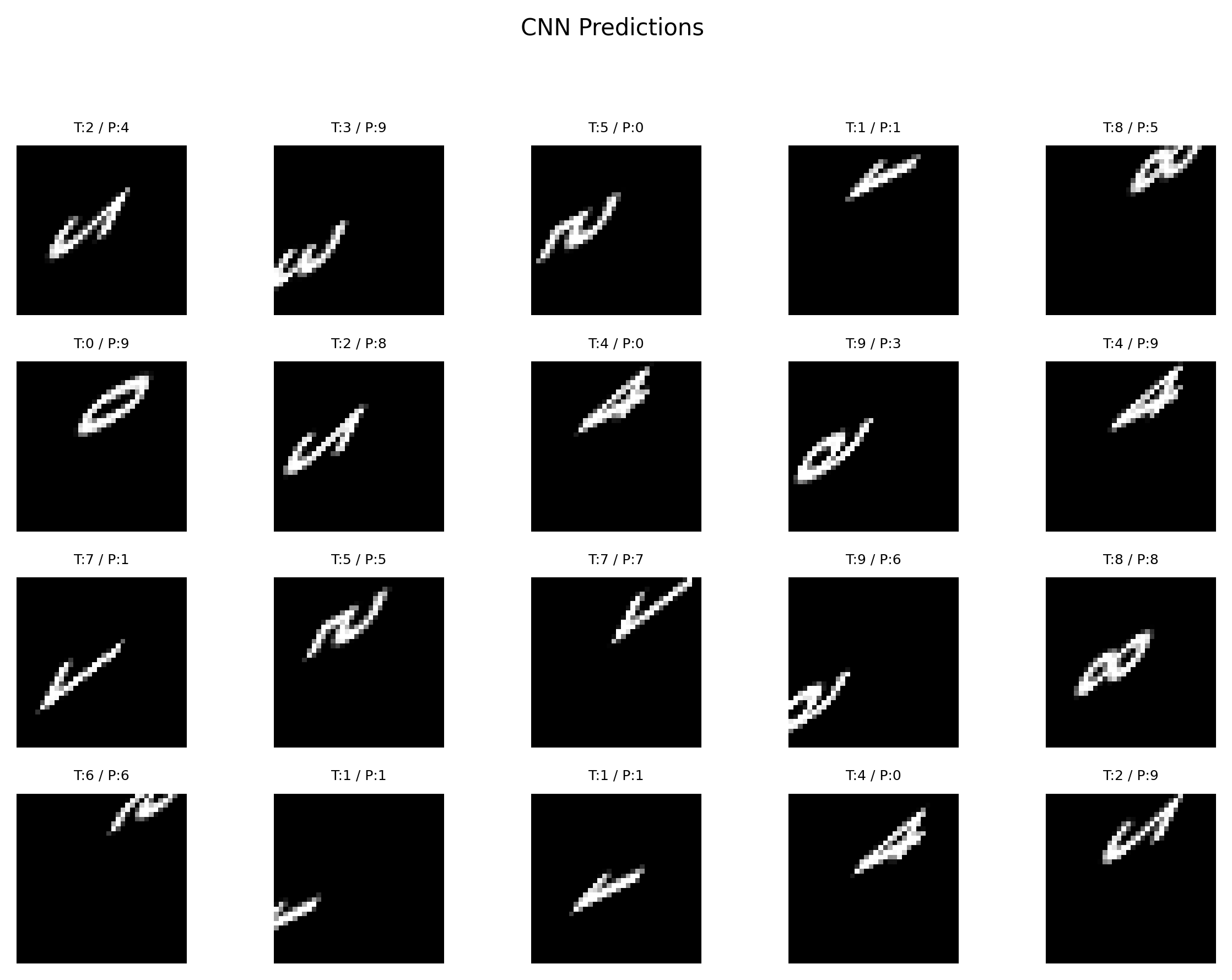}
				\caption{Standard CNN}
				\label{fig:standard_cnn}
			\end{subfigure}
			\caption{Prediction comparison under affine transformation \(A = \begin{bmatrix} 1 & 2 \\ 2 & 1 \end{bmatrix}\). G-CNN outperforms CNN with higher mean accuracy (0.6950 vs. 0.3150).}
			\label{fig:qualitative_comparison}
		\end{figure}


\begin{figure}[htbp]
	\centering
	\begin{subfigure}[b]{0.47\textwidth}
		\centering
		\includegraphics[width=\linewidth]{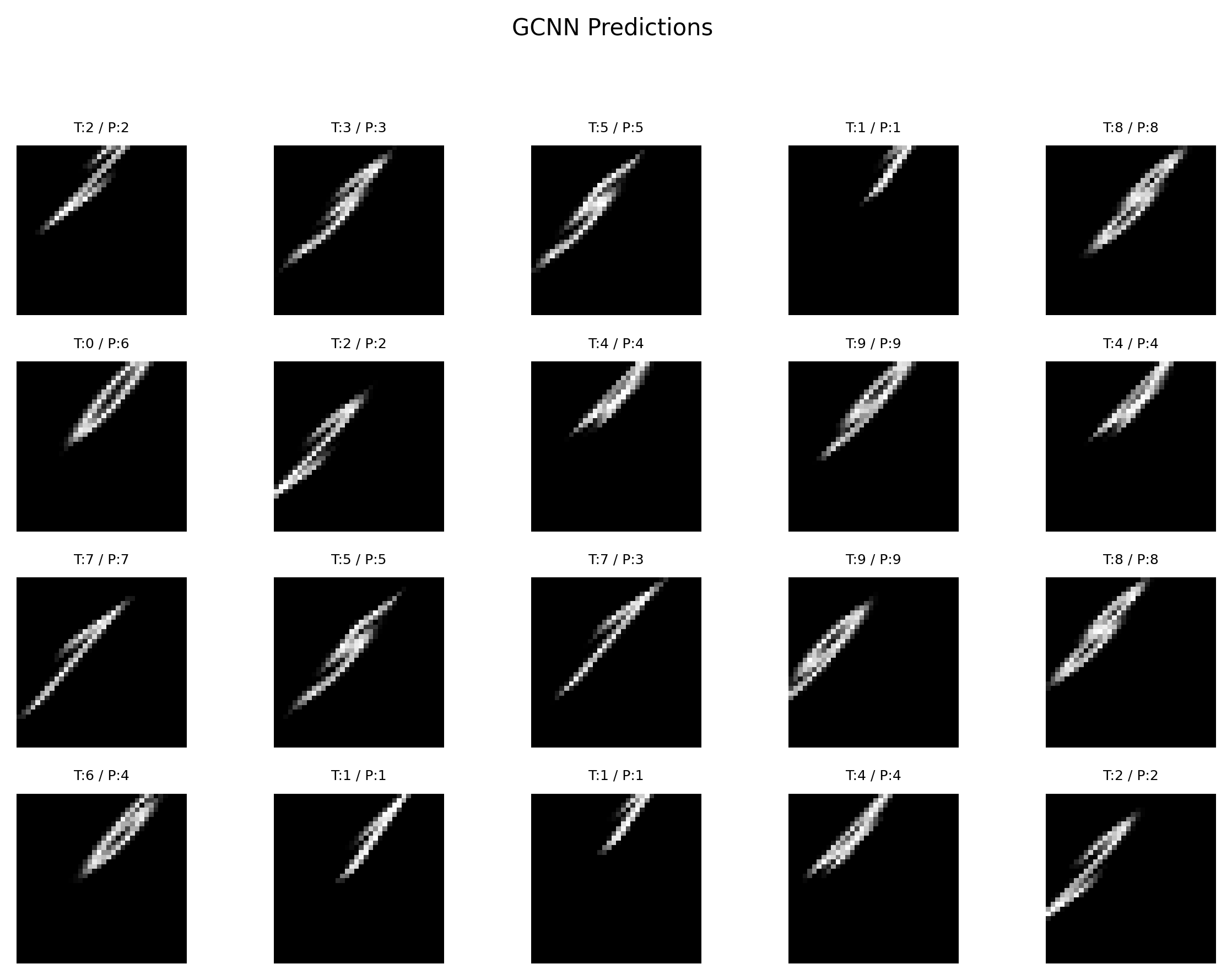}
		\caption{Group Convolution Network}
		\label{fig:group_conv}
	\end{subfigure}
	\hfill
	\begin{subfigure}[b]{0.47\textwidth}
		\centering
		\includegraphics[width=\linewidth]{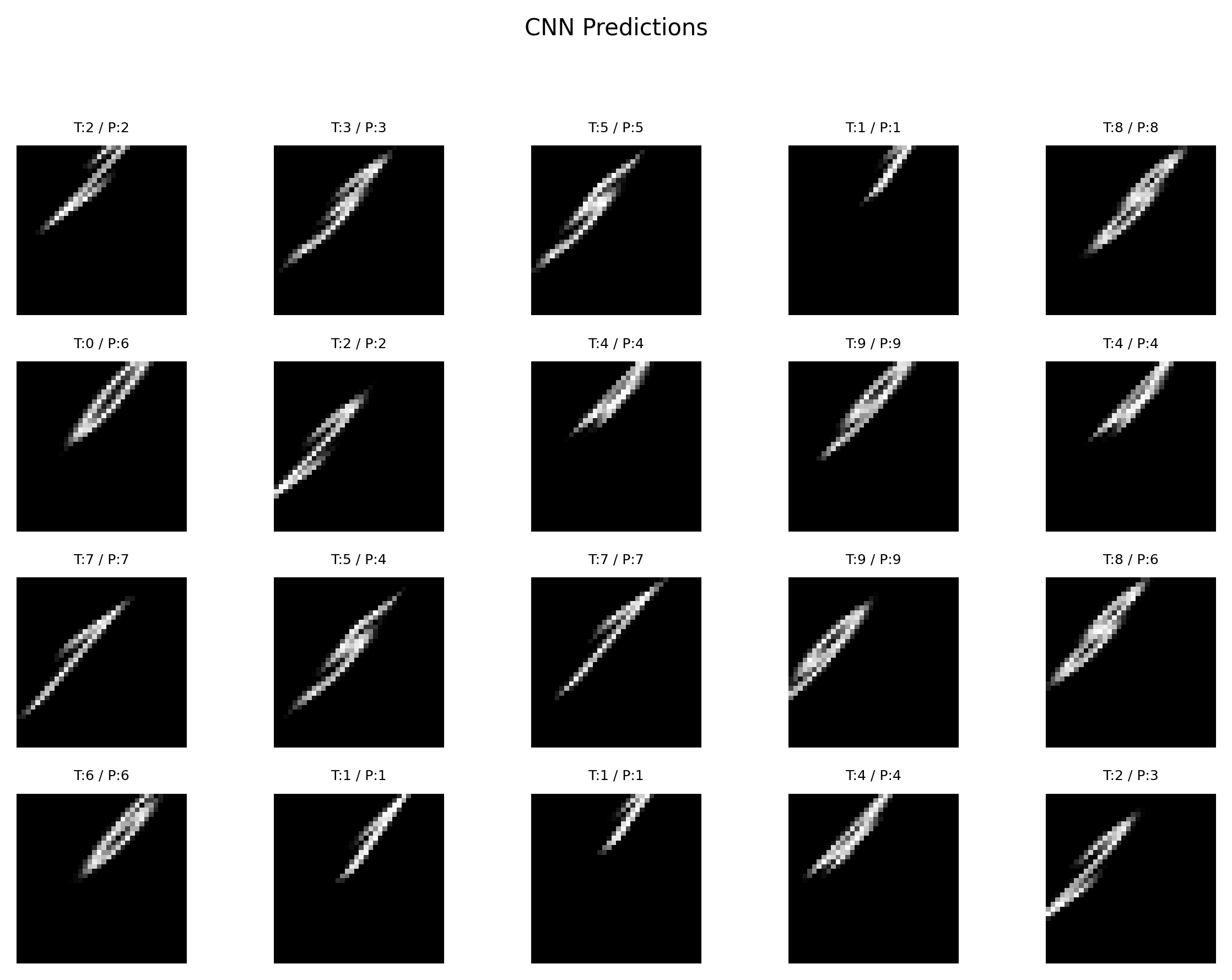}
		\caption{Standard CNN}
		\label{fig:standard_cnn}
	\end{subfigure}
	\caption{Prediction comparison under affine transformation \( A = \begin{bmatrix} 1 & 0.7 \\ 0.7 & 1 \end{bmatrix} \). GCNN outperforms CNN with higher mean accuracy (0.8000 vs. 0.7200).}
	\label{fig:affine_eval2}
\end{figure}

\begin{figure}[htbp]
	\centering
	\begin{subfigure}[b]{0.47\textwidth}
		\centering
		\includegraphics[width=\linewidth]{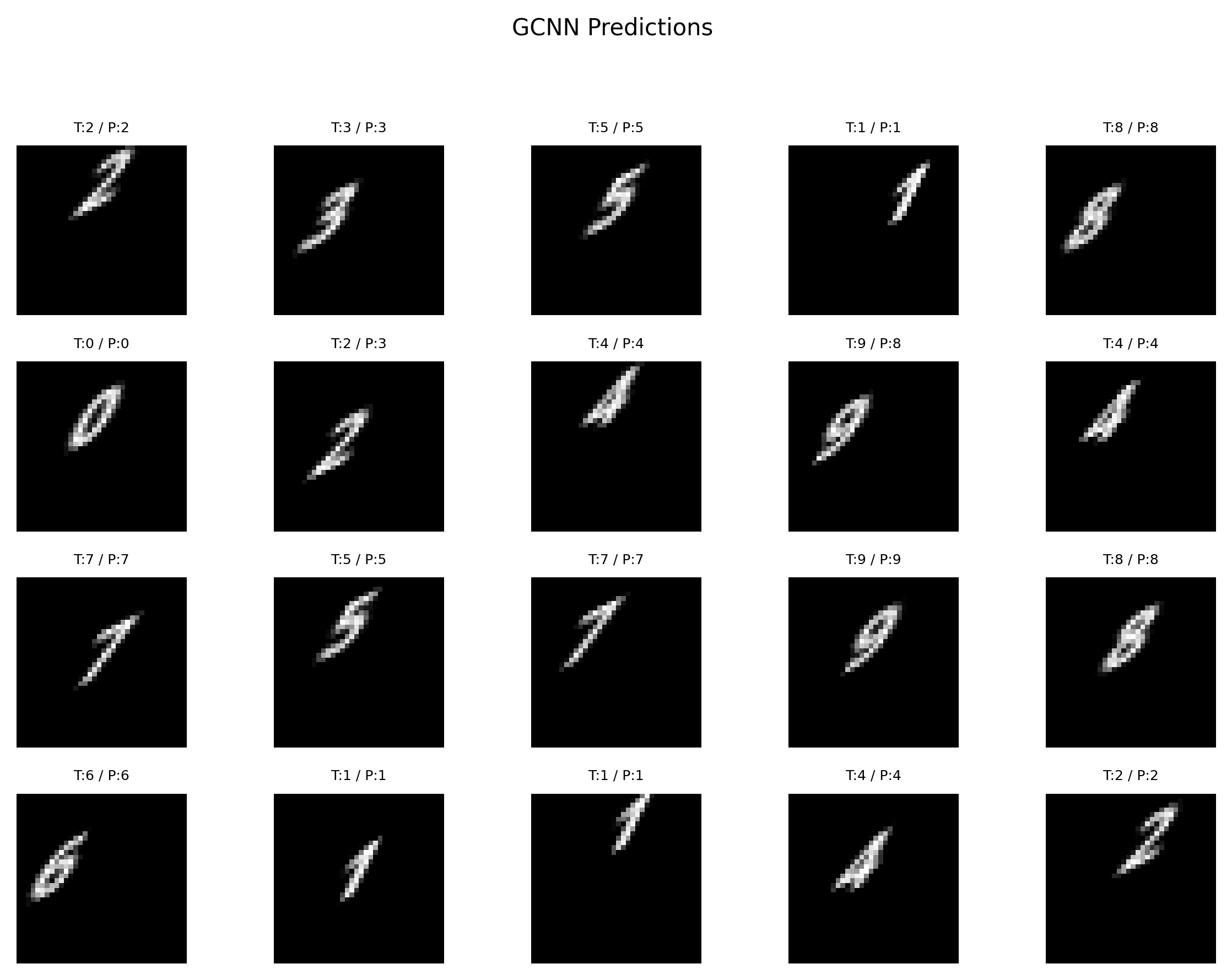}
		\caption{Group Convolution Network}
		\label{fig:group_conv}
	\end{subfigure}
	\hfill
	\begin{subfigure}[b]{0.47\textwidth}
		\centering
		\includegraphics[width=\linewidth]{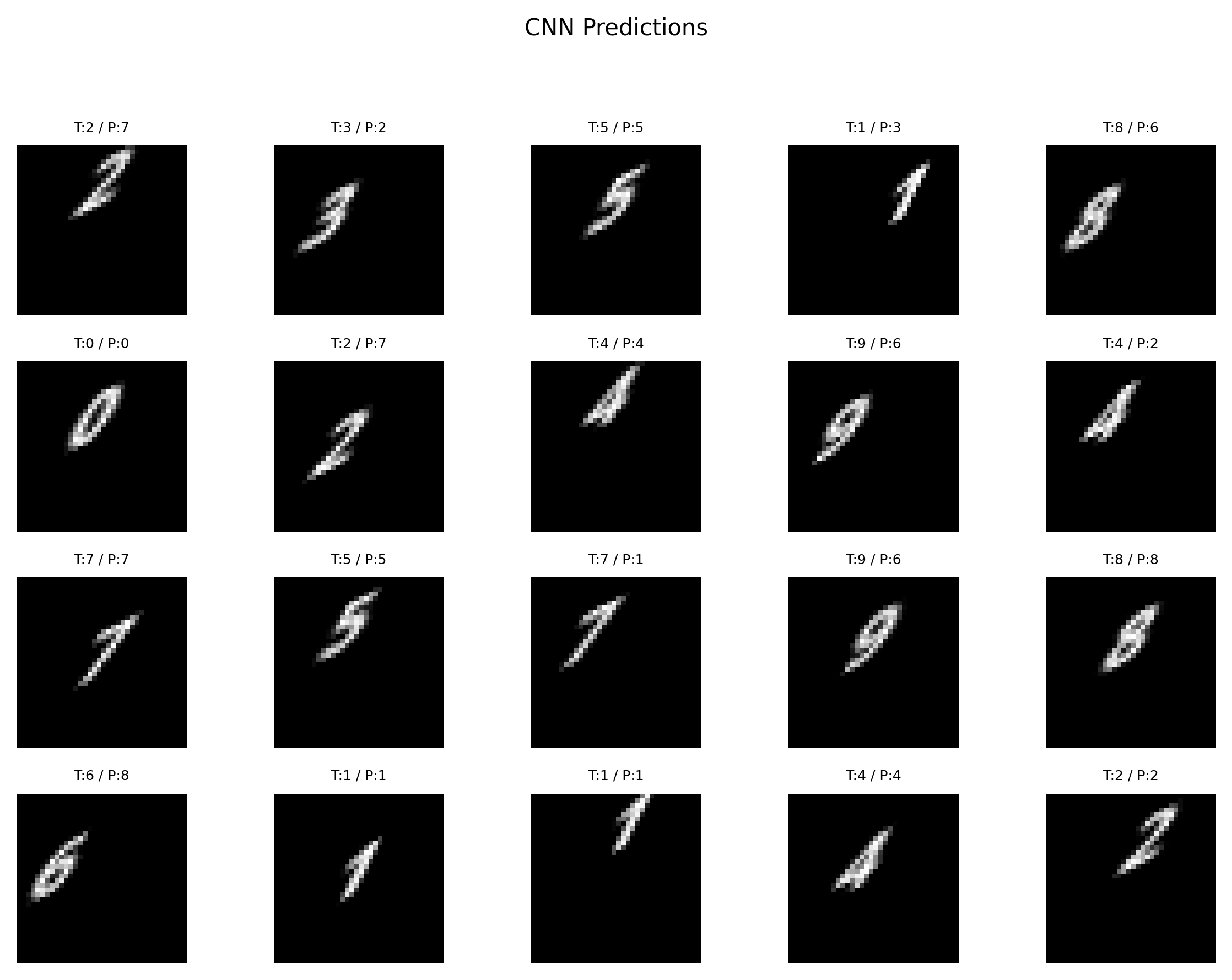}
		\caption{Standard CNN}
		\label{fig:standard_cnn}
	\end{subfigure}
	\caption{Prediction comparison under affine transformation \( A = \begin{bmatrix} 1 & 0.5 \\ 0.5 & 1 \end{bmatrix} \). GCNN outperforms CNN with higher mean accuracy (0.8150 vs. 0.5950).}
\label{fig:affine_eval2}
\end{figure}


\section{Conclusion}

In this work, we have demonstrated that Group Convolutional Neural Networks (G-CNNs) exhibit stability under affine transformations arising from the full general linear group $\mathrm{GL}_2(\mathbb{R})$. This includes a wide range of transformations such as rotations, scalings, and translations, significantly broadening the class of transformations for which G-CNNs are theoretically justified. Our findings represent the first rigorous validation of G-CNN architectures in such a general affine setting.
Furthermore, we showed that the complex group convolutions required for these networks can be simplified to standard integrals over $\mathbb{R}^2$, enhancing computational feasibility. These theoretical insights support the design of more robust and invariant learning models.

Beyond theoretical contributions, our simulations confirm that G-CNNs can outperform conventional CNNs, particularly in data-scarce scenarios. This highlights the practical benefits of incorporating affine invariance into deep learning architectures, potentially expanding their applicability to more complex and realistic pattern recognition tasks where geometric variability is prevalent.

		

\FloatBarrier

\bibliography{sn-bibliography}

		\clearpage
		
		\section{Appendix}
		In this section, we initially present an example analogous to Example \ref{equivalent_inputs} but within the context of $GL(2)$. Then we provide proof of main theorems in the paper.

		\begin{example}
			Consider the below input function:
			\begin{equation*}
				f(\boldsymbol{x})=\begin{cases}
					1 & \text{if $\boldsymbol{x} \in [t_1,t_2] \times [s_1,s_2]$} \\
					0 & \text{otherwise}
				\end{cases},
			\end{equation*} 
			then according to Example (\ref{LiftingExample}) for the lifting of $f$ we have 
			\begin{equation*}
				(\mathcal{K} f)[\boldsymbol{x},\boldsymbol{A}]=\begin{cases}
					\frac{1}{|\det(\boldsymbol{A})|} & \text{if $\boldsymbol{x} \in [t_1,t_2] \times [s_1,s_2]$} \\
					0 & \text{otherwise}
				\end{cases},
			\end{equation*}
			
			as a result
			
			\begin{equation*} 
				\begin{split}
					&\int_{G_2} f([\boldsymbol{x},\boldsymbol{A}])k([\boldsymbol{y},\boldsymbol{B}]^{-1}[\boldsymbol{x},\boldsymbol{A}]) d \mu_{G_2} \\
					&= \int_{G_2} f([\boldsymbol{x},\boldsymbol{A}])k\bigl(\boldsymbol{B}^{-1}\boldsymbol{x}-\boldsymbol{B}^{-1}\boldsymbol{y},\boldsymbol{A}\boldsymbol{B}^{-1}\bigr)d\mu_{G_2} \\
					& = \int_{G_2} \frac{1}{|\det(\boldsymbol{A})|} k\bigl(\boldsymbol{B}^{-1}\boldsymbol{x}-\boldsymbol{B}^{-1}\boldsymbol{y},\boldsymbol{A}\boldsymbol{B}^{-1}\bigr)d\mu_{G_2} \\
					&=\int_{t_1}^{t_2} \int_{s_1}^{s_2}  \int_{\mathrm{GL}_2} \frac{1}{|\det(\boldsymbol{A})|^2}  k(\boldsymbol{B}^{-1}\boldsymbol{x}-\boldsymbol{B}^{-1}\boldsymbol{y},\boldsymbol{A}\boldsymbol{B}^{-1}) d \mu_{\mathrm{GL}_2} dx_1 dx_2,
				\end{split}
			\end{equation*}
			We define the separable kernel as follows
			\begin{equation*}
				k( \boldsymbol{x}, \boldsymbol{P}) = |\det(\boldsymbol{P})|^2 \exp(\langle \boldsymbol{x},\boldsymbol{l}\rangle) \mathcal{N}(\boldsymbol{P}),
			\end{equation*}
			where $\boldsymbol{l} = [l_1,l_2]$, and $\mathcal{N}(\boldsymbol{P})$ is the normalized heat kernel of $GL(2)$. By this choice we obtain

			\begin{equation*} 
				\begin{split}
					&\int_{t_1}^{t_2}\int_{s_1}^{s_2} \int_{\mathrm{GL}_2} \frac{1}{|\det(\boldsymbol{A})|^2}  k(\boldsymbol{B}^{-1}\boldsymbol{x}-\boldsymbol{B}^{-1}\boldsymbol{y},\boldsymbol{A}\boldsymbol{B}^{-1}) dx_1 dx_2  d \mu_{\mathrm{GL}_2} \\
					&= \int_{t_1}^{t_2}\int_{s_1}^{s_2} \int_{\mathrm{GL}_2} \frac{1}{|\det(\boldsymbol{A})|^2} {|\det(\boldsymbol{B^{-1}})|^2}  |\det(\boldsymbol{A})|^2 \exp(\langle \boldsymbol{B}^{-1}\boldsymbol{x}-\boldsymbol{B}^{-1}\boldsymbol{y},\boldsymbol{l}\rangle) \mathcal{N}(\boldsymbol{P})  dx_1 dx_2  d \mu_{\mathrm{GL}_2} \\
					& = \int_{t_1}^{t_2}\int_{s_1}^{s_2} \frac{1}{|\det(\boldsymbol{B})|^2}  \exp(\langle \boldsymbol{B}^{-1}\boldsymbol{x}-\boldsymbol{B}^{-1}\boldsymbol{y},\boldsymbol{l}\rangle) \int_{\mathrm{GL}_2}  \mathcal{N}(\boldsymbol{P})    d \mu_{\mathrm{GL}_2} dx_1 dx_2 \\
					& = \frac{1}{|\det(\boldsymbol{B})|^2 \exp(\langle \boldsymbol{B}^{-1}\boldsymbol{y},\boldsymbol{l}\rangle)}\int_{t_1}^{t_2}\int_{s_1}^{s_2}   \exp(\langle \boldsymbol{B}^{-1}\boldsymbol{x},\boldsymbol{l}\rangle)  dx_1 dx_2 \\
				\end{split}.
			\end{equation*}
			
			Assume that $\boldsymbol{B}^{-1} = \begin{pmatrix}
				\beta_1 & \beta_2 \\
				\beta_3 & \beta_4
			\end{pmatrix}$, then we obtain
			
			\begin{equation*} 
				\begin{split}
					& \frac{1}{|\det(\boldsymbol{B})|^2 \exp(\langle \boldsymbol{B}^{-1}\boldsymbol{y},\boldsymbol{l}\rangle)}\int_{t_1}^{t_2}\int_{s_1}^{s_2}   \exp(\langle \boldsymbol{B}^{-1}\boldsymbol{x},\boldsymbol{l}\rangle)  dx_1 dx_2 \\
					& =\frac{1}{|\det(\boldsymbol{B})|^2 \exp(\langle \boldsymbol{B}^{-1}\boldsymbol{y},\boldsymbol{l}\rangle)}\int_{t_1}^{t_2}\int_{s_1}^{s_2}   \exp(\langle \boldsymbol{B}^{-1}\boldsymbol{x},\boldsymbol{l}\rangle)  dx_1 dx_2 \\
					&=\frac{1}{|\det(\boldsymbol{B})|^2 \exp(\langle \boldsymbol{B}^{-1}\boldsymbol{y},\boldsymbol{l}\rangle)}\int_{t_1}^{t_2}\int_{s_1}^{s_2}   \exp(l_1\beta_1x_1+l_1\beta_2x_2+l_2\beta_3x_1+l_2\beta_4x_2)  dx_1 dx_2 \\  
					&=\frac{1}{|\det(\boldsymbol{B})|^2 \exp(\langle \boldsymbol{B}^{-1}\boldsymbol{y},\boldsymbol{l}\rangle)}\int_{t_1}^{t_2}\int_{s_1}^{s_2}   \exp\Bigl(x_1(l_1\beta_1+l_2\beta_3)+x_2(l_1\beta_2+l_2\beta_4)\Bigr)  dx_1 dx_2 \\ 
					&=\frac{1}{|\det(\boldsymbol{B})|^2 \exp(\langle \boldsymbol{B}^{-1}\boldsymbol{y},\boldsymbol{l}\rangle)(l_1\beta_1+l_2\beta_3)(l_1\beta_2+l_2\beta_4)}\Bigl(\exp\Bigl(t_2(l_1\beta_1+l_2\beta_3)-\exp\Bigl(t_1(l_1\beta_1+l_2\beta_3)\Bigr)\\
					&\times\Bigl(\exp\Bigl(s_2(l_1\beta_2+l_2\beta_4)\Bigr)- \exp\Bigl(s_1(l_1\beta_2+l_2\beta_4)\Bigr) \Bigr).
				\end{split}
			\end{equation*}
			
			For simplicity we can assume $l_1 = l_2 =1$, then we have
			
			\begin{equation*} 
				\begin{split}
					& \frac{1}{|\det(\boldsymbol{B})|^2 \exp(\langle \boldsymbol{B}^{-1}\boldsymbol{y},\boldsymbol{l}\rangle)}\int_{t_1}^{t_2}\int_{s_1}^{s_2}   \exp(\langle \boldsymbol{B}^{-1}\boldsymbol{x},\boldsymbol{l}\rangle)  dx_1 dx_2 \\
					&=\frac{1}{|\det(\boldsymbol{B})|^2 \exp(\langle \boldsymbol{B}^{-1}\boldsymbol{y},[1,1]\rangle)(\beta_1+\beta_3)(\beta_2+\beta_4)}\Bigl(\exp\Bigl(t_2(\beta_1+\beta_3)-\exp\Bigl(t_1(\beta_1+\beta_3)\Bigr)\\
					&\times\Bigl(\exp\Bigl(s_2(\beta_2+\beta_4)\Bigr)- \exp\Bigl(s_1(\beta_2+\beta_4)\Bigr) \Bigr).
				\end{split}
			\end{equation*}
			
		\end{example}
		
		\subsection*{Proof of Theorem \ref{1stlayerinvariance}}
		
		\begin{proof}
			We know that $\int_{\R^2} \frac{k(g^{-1}\boldsymbol{x})f(\boldsymbol{x})}{|\det h|} d\boldsymbol{x} = \int_{\R^2} k(\boldsymbol{x}) f(g\boldsymbol{x})d\boldsymbol{x}$, then we have

			\begin{equation*}
				\begin{aligned} & \sup _{g^{\prime}}\bigl|\left( (\mathcal{K}f_1) -\rho\left(g^{-1}\right) (\mathcal{K}f_1)\right)\left(g^{\prime}\right)\bigr|  \\
					& =\sup _{g^{\prime}}\biggl| \int_{\R^2}k(\boldsymbol{x})f_1(g'\boldsymbol{x}) d\boldsymbol{x} -k(\boldsymbol{x})f_1(gg'\boldsymbol{x}) d\boldsymbol{x}\biggr| \\
					&\le \sup _{g^{\prime}} \int_{\R^2} |k(\boldsymbol{x})| |f_1(g'\boldsymbol{x})-f_1(gg'\boldsymbol{x})|,
				\end{aligned}
			\end{equation*}
			by setting $g'\boldsymbol{x}=\boldsymbol{y}$ for the last term in above we have
			\begin{equation*}
				\sup _{g^{\prime}} \int_{\R^2} |k(\boldsymbol{x})| |f_1(g'\boldsymbol{x})-f_1(gg'\boldsymbol{x})| 
				\le \epsilon \int_{\R^2} |k(\boldsymbol{x})| d\boldsymbol{x} = \epsilon \|k\|_1^{\R^2}.
			\end{equation*}
			
		\end{proof}

		
		\subsection*{Proof of Theorem \ref{2ndlayerinvariance}}
		
		\begin{proof}
			We have 
			\begin{equation*}
				\begin{aligned}
					\| (\mathcal{K}f_1)*k - \rho(\tilde{h}) (\mathcal{K}f_2)*k\|_{\sup}^{G_2} &= \sup \Bigr|\int_{G_2} (\mathcal{K}f_1)(g) k(h^{-1}(g)) - \rho(\tilde{h}) (\mathcal{K}f_2)(g) k(h^{-1}g) d\mu_{G_2}(g)\Bigl|  \\
					&\le \sup \int_{G_2} \Bigr| (\mathcal{K}f_1)(g) k(h^{-1}(g)) - \rho(\tilde{h}) (\mathcal{K}f_2)(g) k(h^{-1}g) \Bigl| d\mu_{G_2}(g)  \\
					&\le \sup \int_{G_2} \Bigr| (\mathcal{K}f_1)(g)  - \rho(\tilde{h}) (\mathcal{K}f_2)(g)  \Bigl| \Bigr|k(h^{-1}(g))\Bigl| d\mu_{G_2}(g) \\
					&\le \epsilon \|k\|_1^{G_2}.
				\end{aligned}   
			\end{equation*}
			The second part of the theorem
			results by selecting $k=\delta(g-h')$.
		\end{proof}

		
		\subsection*{Proof of Theorem \ref{3rdlayerinvariance}}
		
		\begin{proof}
			We know that 
			\begin{equation*}
				\begin{aligned}
					&\Bigr|\int_{G_2} \bigr((\mathcal{K}f_1)*k -(\mathcal{K}f_2)*k\bigl)(h) d\mu_{G_2}(h) \Bigl|=\\&\Bigr|\int_{G_2}\int_{G_2} \bigl((\mathcal{K}f_1)(g)k(h^{-1}g) \bigr) d\mu_{G_2}(g) d\mu_{G_2}(h) - \int_{G_2}\int_{G_2} \bigl((\mathcal{K}f_2)(g)k(h^{-1}g) \bigr) d\mu_{G_2}(g) d\mu_{G_2}(h) \Bigl|.
				\end{aligned}
			\end{equation*}
			Then for the second term in the above equation we have and replacing $g$ with $\tilde{h}^{-1}g$ we have 
			
			\begin{equation*}
				\begin{aligned}
					&\int_{G_2}\int_{G_2} \bigl((\mathcal{K}f_2)(g)k(h^{-1}g) \bigr) d\mu_{G_2}(g) d\mu_{G_2}(h)    \\
					&= \int_{G_2}\int_{G_2} \bigl((\mathcal{K}f_2)(\tilde{h}^{-1}g)k(h^{-1}\tilde{h}^{-1}g) \bigr) d\mu_{G_2}(g) d\mu_{G_2}(h) \\
					&= \int_{G_2}\int_{G_2} \bigl((\mathcal{K}f_2)(\tilde{h}^{-1}g)k((\tilde{h}h)^{-1}g) \bigr) d\mu_{G_2}(g) d\mu_{G_2}(h),
				\end{aligned}
			\end{equation*}
			if we set 
			\begin{equation*}
				f(h)= \int_{G_2} \bigl((\mathcal{K}f_2)(\tilde{h}^{-1}g)k((\tilde{h}h)^{-1}g) \bigr) d\mu_{G_2}(g),
			\end{equation*}
			then
			\begin{equation*}
				\begin{aligned}
					&\int_{G_2}\int_{G_2} \bigl((\mathcal{K}f_2)(\tilde{h}^{-1}g)k((\tilde{h}h)^{-1}g) \bigr) d\mu_{G_2}(g) d\mu_{G_2}(h) \\
					&=\int_{G_2} f(h) d\mu_{G_2}(h)= \int_{G_2} f(\tilde{h}h) d\mu_{G_2}(h)   \\
					&= \int_{G_2}\int_{G_2} \bigl((\mathcal{K}f_2)(\tilde{h}^{-1}g)k(h^{-1}g) \bigr) d\mu_{G_2}(g) d\mu_{G_2}(h),
				\end{aligned}
			\end{equation*}
			therefore, 
			
			\begin{equation*}
				\begin{aligned}
					&\Bigr|\int_{G_2}\int_{G_2} \bigl((\mathcal{K}f_1)(g)k(h^{-1}g) \bigr) d\mu_{G_2}(g) d\mu_{G_2}(h) - \int_{G_2}\int_{G_2} \bigl((\mathcal{K}f_2)(g)k(h^{-1}g) \bigr) d\mu_{G_2}(g) d\mu_{G_2}(h) \Bigl| =\\
					&\Bigr|\int_{G_2}\int_{G_2} \bigl((\mathcal{K}f_1)(g)k(h^{-1}g) \bigr) d\mu_{G_2}(g) d\mu_{G_2}(h) - \int_{G_2}\int_{G_2} \bigl((\mathcal{K}f_2)(\tilde{h}^{-1}g)k(h^{-1}g) \bigr) d\mu_{G_2}(g) d\mu_{G_2}(h) \Bigl| =\\
					& \Bigr|\int_{G_2}\int_{G_2} \bigl((\mathcal{K}f_1)(g)-(\mathcal{K}f_2)(\tilde{h}^{-1}g)\bigr) k(h^{-1}g) d\mu_{G_2}(g) d\mu_{G_2}(h)  \Bigl| = \Bigr|\int_{G_2} \bigl((\mathcal{K}f_1)-(\mathcal{K}f_2)\circ \tilde{h}^{-1}\bigr)* k \,d\mu_{G_2}(h)  \Bigl|\\
					& \le \int_{G_2} \Bigr|\bigl((\mathcal{K}f_1)-(\mathcal{K}f_2)\circ \tilde{h}^{-1}\bigr)* k \,\Bigl| d\mu_{G_2}(h) = \Bigl\|\bigl((\mathcal{K}f_1)-(\mathcal{K}f_2)\circ \tilde{h}^{-1}\bigr)* k \,\Bigr\|_1^{G_2} \le \epsilon \|k\|_1^{G_2}.
				\end{aligned}
			\end{equation*}

		\end{proof}
		
		\subsection*{Proof of Theorem \ref{G2integration}}
		\begin{proof}
			We know that
			\begin{equation*} 
				\begin{split}
					&\int_{G_2} f([\boldsymbol{x},\boldsymbol{A}])k([\boldsymbol{y},\boldsymbol{B}]^{-1}[\boldsymbol{x},\boldsymbol{A}]) d \mu_{G_2} \\
					&= \int_{G_2} f([\boldsymbol{x},\boldsymbol{A}])k\bigl(\boldsymbol{B}^{-1}\boldsymbol{x}-\boldsymbol{B}^{-1}\boldsymbol{y},\boldsymbol{A}\boldsymbol{B}^{-1}\bigr)d\mu_{G_2}.
				\end{split}
			\end{equation*}
			Employing (\ref{Eq:GnDecomposition})
			we have
			
			\begin{equation*} 
				\begin{split}
					&\int_{G_2} f([\boldsymbol{x},\boldsymbol{A}])k(\boldsymbol{B}^{-1}\boldsymbol{x}-\boldsymbol{B}^{-1}\boldsymbol{y},\boldsymbol{A}\boldsymbol{B}^{-1})d\mu_{G_2} \\
					&=\int_{\mathrm{GL}_2} \int_{\mathbb{R}^2} f[\boldsymbol{x}, \boldsymbol{A}]k(\boldsymbol{B}^{-1}\boldsymbol{x}-\boldsymbol{B}^{-1}\boldsymbol{y},\boldsymbol{A}\boldsymbol{B}^{-1}) \frac{dx_1 dx_2 }{|\operatorname{det}(\boldsymbol{A})|} d \mu_{\mathrm{GL}_2},        
				\end{split}
			\end{equation*}
			we also set 
			
			\begin{equation}\label{H_A}
				H_{f,k}(\boldsymbol{A},\boldsymbol{B},\boldsymbol{y}) :=  \int_{\mathbb{R}^2} f[\boldsymbol{x}, \boldsymbol{A}]k(\boldsymbol{B}^{-1}\boldsymbol{x}-\boldsymbol{B}^{-1}\boldsymbol{y},\boldsymbol{AB}^{-1}) \frac{dx_1 dx_2 }{|\operatorname{det}(\boldsymbol{A})|}.  
			\end{equation}
			From separability property of kernel we have $k(\boldsymbol{x},\boldsymbol{A}) = k_1(\boldsymbol{x})k_2(\boldsymbol{A})$. As a result
			\begin{equation}\label{H_A}
				\begin{split}
					H_{f,k}(\boldsymbol{A},\boldsymbol{B},\boldsymbol{y}) &=  \int_{\mathbb{R}^2} f[\boldsymbol{x}, \boldsymbol{A}]k_1(\boldsymbol{B}^{-1}\boldsymbol{x}-\boldsymbol{B}^{-1}\boldsymbol{y})k_2(\boldsymbol{A}\boldsymbol{B}^{-1}) \frac{dx_1 dx_2 }{|\operatorname{det}(\boldsymbol{A})|} \\
					& = \frac{k_2(\boldsymbol{A}\boldsymbol{B}^{-1}) }{{|\operatorname{det}(\boldsymbol{A})|}} \int_{\mathbb{R}^2} f[\boldsymbol{x}, A]k_1(\boldsymbol{B}^{-1}\boldsymbol{x}-\boldsymbol{B}^{-1}\boldsymbol{y})dx_1 dx_2 \\
					& = \frac{k_2(\boldsymbol{A}\boldsymbol{B}^{-1}) }{{|\operatorname{det}(\boldsymbol{A})|}} \Bigl(f*(k_1 \circ \boldsymbol{B}^{-1})\Bigr) \\
					& = \frac{k_2(\boldsymbol{A}\boldsymbol{B}^{-1}) }{{|\operatorname{det}(\boldsymbol{A})|}} \mathcal{F}^{-1}\Bigl(\mathcal{F}(f) \mathcal{F}(k_1 \circ \boldsymbol{B}^{-1})\Bigr),
				\end{split}     
			\end{equation}
			where $\mathcal{F}(\cdot)$  denotes the Fourier transform. The next step is to find an explicit form for the  Fourier transform. We can apply the result from \citep{bracewell1993affine}. Assume that $\mathcal{F}(k_1)= K_1(\boldsymbol{u})$ and $\mathcal{F}(f)= F(\boldsymbol{u})$ then we have
			
			\begin{equation*}
				H_{f,k}(\boldsymbol{A},\boldsymbol{B},\boldsymbol{y})  =   \frac{k_2(\boldsymbol{A}\boldsymbol{B}^{-1}) }{{|\operatorname{det}(\boldsymbol{A})||\operatorname{det}(\boldsymbol{B}^{-1})|}} \mathcal{F}^{-1}\Bigl(F(\boldsymbol{u}) k_1(\boldsymbol{B}^{\top}\boldsymbol{u}  ) \Bigr).
			\end{equation*}
			Now we use decomposition of $\mathrm{GL}_2(\R)$ as $K_0 \ltimes H(1,0)$ in \citep{schindler1993iwasawa,milad2023harmonic}.

			

			
			

			\begin{proposition}
				[Proposition 5.1 of \citep{milad2023harmonic}] If $\boldsymbol{A}=\left(\begin{array}{ll}a & b \\ c & d\end{array}\right) \in \mathrm{GL}_2(\mathbb{R})$, then $\boldsymbol{A}$ can be uniquely decomposed as the product $\boldsymbol{A=M_A C_A}$ with $\boldsymbol{M_A} \in K_0$ and $\boldsymbol{C_A} \in H_{(1,0)}$. In fact
				$$
				\boldsymbol{M_A}=\left(\begin{array}{cc}
					s & -t \\
					t & s
				\end{array}\right), \text { with } \ s=\frac{d(a d-b c)}{b^2+d^2}, t=\frac{-b(a d-b c)}{b^2+d^2},
				$$
				and
				$$
				\boldsymbol{C_A}=\left(\begin{array}{cc}
					1 & 0 \\
					u & v
				\end{array}\right), \text { with } \ u=\frac{c d+a b}{(a d-b c)}, v=\frac{b^2+d^2}{(a d-b c)}.
				$$
				This factorization leads to a parallel factorization of $G_2$.     
			\end{proposition}

			

			
			
			
			
			
			Consider the one to one transform between $H$ and $H^*$ so that $H^*(s,t,u,v,\boldsymbol{B},\boldsymbol{y}) := H_{f,k}(a,b,c,d,\boldsymbol{B},\boldsymbol{y})$, where $a = s-ut$, $c =t+us $, $b =-t/v $, and $d = s/v $. Employing the above proposition and Equation (\ref{GL_Decomposition}) we can write 
			
			
			
			\begin{equation*}
				\begin{aligned}
					H'(\boldsymbol{B},\boldsymbol{y})  &=\int_{\mathrm{GL}_2} H_{f,k}(\boldsymbol{A},\boldsymbol{B},\boldsymbol{y}) d \mu_{\mathrm{GL}_2} \\
					& = \int_{\mathrm{GL}_2} H^*(s(a,b,c,d),t(a,b,c,d),u(a,b,c,d),v(a,b,c,d),\boldsymbol{B},\boldsymbol{y}) d \mu_{\mathrm{GL}_2}.
				\end{aligned}
			\end{equation*}
			Therefore, we obtain
			
			\begin{equation*}
				\int_{\mathrm{GL}_2} H_{f,k}(\boldsymbol{A},\boldsymbol{B},\boldsymbol{y}) d \mu_{\mathrm{GL}_2} = \int_{K_0} \int_{H_{(1,0)}} H^*(s,t,u,v,\boldsymbol{B},\boldsymbol{y})|v| d \mu_{H_{(1,0)}} d \mu_{K_0},
			\end{equation*}
			as $\det(\boldsymbol{C_A}) = |v|$.
			Then we define 
			
			\begin{equation*}
				\begin{split}
					H^*(s,t,\boldsymbol{B},\boldsymbol{y})&=\int_{H_{(1,0)}} H^*(s,t,u,v,\boldsymbol{B},\boldsymbol{y})\det(\boldsymbol{C_A}) d \mu_{H_{(1,0)}}(u,v) \\
					&=\int_{G_1} H^*(s,t,u,v,\boldsymbol{B},\boldsymbol{y})\det(\boldsymbol{C_A}) d \mu_{G_1}(u,v) \\
					&= \int_{\R} \int_{\R} H^*(s,t,u,v,\boldsymbol{B},\boldsymbol{y})\frac{du dv}{|v| }.
				\end{split} 
			\end{equation*}
			The next step is to compute integration of $H^*(s,t,\boldsymbol{B},\boldsymbol{y})$ over $K_0$, which is equal to 
			
			\begin{equation}
				\int_{\R} \int_{\R} H^*(s,t,\boldsymbol{B},\boldsymbol{y}) \frac{ds dt}{s^2+t^2}.
			\end{equation}
		\end{proof}

\end{document}